\documentclass[journal]{IEEEtran}

\usepackage[T1]{fontenc}

\usepackage{amsmath}
\usepackage{amsthm}
 
\usepackage[cmintegrals]{newtxmath}

\usepackage{amssymb,mathrsfs}
\usepackage{bm}
\usepackage{bbm}
\usepackage{subfigure}
\usepackage{graphicx}
\usepackage{caption}
\usepackage{subfigure}
\usepackage{calc}
\usepackage{epstopdf}
\usepackage{caption}
\usepackage{subfigure}
\newtheorem{thm}{Theorem}
\newtheorem{lm}{Lemma}

\newtheorem{df}{Definition}
\newtheorem{cond}{Assumption}
\usepackage{verbatim}
\usepackage[ruled,vlined]{algorithm2e}
\usepackage{array}
\usepackage{cite}

\usepackage{xcolor}
\usepackage{floatrow}
\usepackage{lipsum}

\usepackage{algorithmic}

\begin{document}


\title{Statistical Robust Chinese Remainder Theorem for Multiple Numbers: Wrapped Gaussian Mixture Model}

\author{Nan Du, Zhikang Wang and Hanshen Xiao 

\thanks{Nan Du is with Department of Statistics, Harvard University, USA. E-mail: nandu@mit.edu}
\thanks{Zhikang Wang is with the Department of Physics, University of Tokyo, 7-3-1 Hongo, Bunkyo-ku, Tokyo, Japan. E-mail: wang@cat.phys.s.u-tokyo.ac.jp }
\thanks{Hanshen Xiao is with CSAIL and the EECS Department, MIT, Cambridge, USA. E-mail: hsxiao@mit.edu.}
}



\let\lc\langle
\let\rc\rangle

\date{}
\date{}

\maketitle

\begin{abstract}
Generalized Chinese Remainder Theorem (CRT) has been shown to be a powerful approach to solve the ambiguity resolution problem. However, with its close relationship to number theory, study in this area is mainly from a coding theory perspective under deterministic conditions. Nevertheless, it can be proved that even with the best deterministic condition known, the probability of success in robust reconstruction degrades exponentially as the number of estimand increases. In this paper, we present the first rigorous analysis on the underlying statistical model of CRT-based multiple parameter estimation, where a generalized Gaussian mixture with background knowledge on samplings is proposed. To address the problem, two novel approaches are introduced. One is to directly calculate the conditional maximal a posteriori probability (MAP) estimation of residue clustering, and the other is to iteratively search for MAP of both common residues and clustering. Moreover, remainder error-correcting codes are introduced to improve the robustness further. It is shown that this statistically based scheme achieves much stronger robustness compared to state-of-the-art deterministic schemes, especially in low and median Signal Noise Ratio (SNR) scenarios. 
\end{abstract}


\begin{section}{Introduction}

\noindent There are many practical scenarios in which measurements of a physical quantity, such as frequency and phase, are separately and distributively conducted on different devices. The well-known ambiguity resolution problem is one of the potential challenges lying in these circumstances. Due to hardware limitation, some measurements result in unavoidable modulo operation when the samples are obtained.

To be more specific, when an exponential signal 
is undersampled, 
after Fourier transform, it can only tell the residues of frequency modulo sampling rate from the observed peaks [10]-[13], [18]-[26]. Apart from frequency ambiguity, range ambiguity also occurs in phase-based distance measurement \cite{xia2007phase}. For example, in a pulse repetition frequency (PRF) radar \cite{mle2}, when the distance exceeds the wavelength of transmit pulses, the raw signal from a reflection is the distance modulo the used wave length. As two fundamental tasks, frequency and phase estimation have a wide range of applications, such as localization in wireless sensor networks \cite{mle2} \cite{wang2011robust} and imaging of moving targets in synthetic aperture radar (SAR) \cite{SAR1},\cite{SAR2},\cite{SAR3} \cite{SAR4}.

Although Chinese Remainder Theorem (CRT) is a straightforward approach to deal with the relation between integers and residues, conventional CRT cannot be trivially applied to solve the above problems. The underlying difficulty is twofold: ambiguity and errors. Firstly, in the model, we need to reconstruct several real numbers simultaneously. As the residues are observed from unordered sample sets, the correspondence between the residues and the numbers is missing. The correspondence ambiguity further strengthens the moduli operation ambiguity. Secondly, errors. Errors not only makes the correspondence determination harder; on the other hand, because the representation of a number with its residues is non-weighted, error control in CRT is more complicated than weighted systems. A small error in a residue can cause incredibly large reconstruction deviation in traditional CRT. To this end, in the last two decades there has been a great deal of research work contributed to overcome the two above obstacles, and the development of robust and generalized CRT has had tremendous progress, which makes it applicable to many higher level problems.


\begin{itemize}
\item \textbf{Ambiguity:} The largest dynamic range $D$ of $Y_i$s needs to be found, given a group of moduli $\{m_1, m_2, ... ,m_L\}$ ignoring statistical errors. It is basically a geometric algebra problem that for all possible $\{Y_1,Y_2,...,Y_N\}, Y_i \in [0,D)$, we ask for a maximum $D$ while ensuring no ambiguity in the residue sets
i.e., their residue sets are distinct. The first lower bound of $D$ in general was given in \cite{dynamic-1}, and it was sharpened further in \cite{dynmaic-sharpen}, and also analyzed in \cite{dynamic-2} under an additional assumption that the residues from each sampler are distinct. So far, the optimal bound of $D$ is only known for $N=2$ case \cite{dynmaic-two}. Another direction in this field is to explore efficient decoding schemes to recover $Y_i$. One trivial way is to enumerate all possibility of residue combinations, and to pick one from each to implement CRT. Clearly it encounters exponential complexity $O(N^L)$. However, under the assumption of distinct residues, polynomial time algorithm also exists, as in \cite{dynamic-2}, \cite{TSP2018}, based on finding roots of an integer-coefficient polynomial using LLL algorithm.

\item \textbf{Error:} For the one integer case, i.e., $N=1$, when the residues bear errors, the error control techniques involved are relevant to error correcting codes. After a long history of remainder code study \cite{rns1963}, \cite{krishna1993theory}, \cite{jenkins1984technique}, the first polynomial time decoding scheme was proposed by Goldreich et al. \cite{goldreich1999} and the correction capacity in the list decoding scenario was improved further by Boneh \cite{boneh2002} and Guruswami et al. \cite{guruswami2000}. From the classic coding theory perspective, the above-mentioned researches only focus on the errors measured in Hamming weights. While, in our model, small errors exist in every measurement, and clearly using infinite norm to describe the errors should be better. The first error control with a bound of error magnitude was presented in \cite{xia2007phase} using a searching-based reconstruction method. Some follow-up work \cite{li2008fast}, \cite{li2009fast} further reduced the dimension of searching till the first closed-form robust CRT (RCRT) proposed in \cite{closed}. Generalized version of the closed form RCRT can be found in \cite{multistage}, \cite{xu2014solving}. In general, adding redundancy is necessary. The difference with a Hamming-weighted error scenario is that, when errors have bounded magnitude, using non-coprime moduli is more feasible. In \cite{sp2017}, it was proved that the most efficient way is to select moduli such that they both globally and pairwise, share a same greatest common factor (gcd), i.e., in a form $m_l = \Gamma M_l$, where $M_l$s are co-prime $l=1,2,...,L$. In addition, different from exact error correction, when errors are measured in an infinite norm, the best performance we may expect is that the deviation between $Y_i$ and the estimation $\hat{Y}_i$, i.e., $|\hat{Y}_i-Y_i|$, is no bigger than the maximum error magnitude, which we will address formally later.
\end{itemize}

To overcome the above two challenges simultaneously, a special case of $N=2$ was solved in \cite{RCRT-two} and the generalized solution was presented later in \cite{TSP2018}. The idea behind \cite{TSP2018} is made up of two parts. Still assuming that $m_l$ is in a form $\Gamma M_l$, under an error bound of $\frac{\Gamma}{4N}$, the problem can be reduced to an errorless CRT to approximate the folding number $\lfloor \frac{Y_i}{\Gamma} \rfloor$. Secondly, by implementing Generalized CRT for multiple numbers, $\lfloor \frac{Y_i}{\Gamma} \rfloor$ can be approximated, and the correspondence between $Y_i$ and residues is further determined. Hence the rest work for reconstruction is trivial. 

In spite of deterministic assumption of errors, the problem has not been systematically studied in statistics. To the best of our knowledge, in the line of research respect to statistical CRT, the maximum likelihood estimation has only been explored in \cite{mle1}, \cite{mle2} for a single number. It is noted that, when $N \geq 2$, almost all existing work figures out the correspondence ambiguity of residues with the help of redundancy. In \cite{TSP2018}, it is required that the least common multiple (lcm) of all moduli should be in the order of the product of $Y_i$, i.e, 
\begin{equation}
\label{dynamic}
\Gamma \prod_{l=1}^{L} M_l = O(D^N),
\end{equation}
where $D$ is still the dynamic range of $Y_i$s. With a fixed error bound, when more moduli are added for sufficient coding redundancy to deal with a larger $N$, \footnote{From (\ref{dynamic}), the number of moduli required is almost proportional to $N$ to deterministically determine the correspondence ambiguity.} the probability that all errors fall into one admissible range decreases exponentially. Thereby, surprisingly, using more moduli and samples will degrade the estimation performance, and this downside is more evident when SNR is low. Here, our motivation is whether such gap can be alleviated using statistical inference. We are inspired by the fact that, when the correspondence between $Y_i$s and residues is obtained, the problem is reduced to $N$ independent RCRT for each $Y_i$, and thus we only require the lcm of all moduli to be bigger than $D$, i.e., $\Gamma \prod_{l=1}^{L} M_l >D$. Different from deterministic methods that leverage coding redundancy, we show that with statistical techniques, the correspondence between residues and $Y_i$s can be estimated more efficiently. Our contribution is summarized as below. 

\begin{itemize}
\item We show that RCRT for multiple numbers can be reduced to a generalized wrapped Gaussian Mixture Model (GMM) with extra information on sampling. A systematical statistical analysis with further explanation in its complexity is presented.
\item We propose two algorithms to address the problem. In Algorithm 1, we first derive the MAP of classification on residues under Assumption 1 in a semi-closed form and the problem is thus reduced to $N$ conventional RCRT. Inspired by \cite{mle1} which applies MLE-based RCRT, we further propose Algorithm 2 as an iterative scheme to find out the MAP of both $Y_i$ and classification in general.
\item We show that the tradeoff amongst the four primary parameters, $N$, $L$, $\Gamma$, and SNR, can be further improved by incorporating with error-correcting codes. With extensive simulation results, it is shown that the statistically based scheme significantly improves performance when compared to deterministic methods, especially for low and median SNR cases. For very high SNR cases the deterministic methods may outperform ours, which is consistent with the theoretical analysis. 
\end{itemize}

The rest of our paper is organized as follows. In Part II, the background and a skeleton of our proposed scheme are presented. Part III develops a framework of generalized GMM with prior on sampling, and an EM-based scheme is proposed to approximate the MLE. In Part IV, the MAP of clustering is analyzed, and we prove the optimal solution can be expressed in a semi closed-form under Assumption \ref{cond}. In Part V, the simulation results of the performance comparison and parameter trade-off are presented. Remainder codes are further introduced to strengthen the clustering. Conclusions and future prospects follow at last.  

\end{section}

\section{Background}
\noindent First we restate the problem formally here. In the task, there are $N$ real numbers, denoted by $Y_1, Y_2, ... ,Y_N$, independently and uniformly distributed in $[0,D)$, where $D$ is a predetermined positive number, termed as the dynamic range. For the $l^{th}$ sampling, with a fixed modulus $m_l$, an un-ordered sample set formed by the residues of each $Y_i$ modulo $m_l$ buried with errors is obtained.

 
 Before we start, another important issue that needs quantifying is, what is meant by $robustness$ and what is the best performance we may expect from the assumption. It is noted that, when $N=1$ and $m_l \gg Y_1$, for each $l \in \{1,2,...,L\}$, the residues obtained should be of $Y_1$ itself with noise without modulo operation. Therefore, given an error bound $\delta = \max_{il} |\Delta_{il}|$, the best performance we can expect will be that the reconstruction error $|\hat{Y}_i - Y_i|$ is no bigger than $\delta$. 

Unfortunately, when modulo operations are involved and moduli are relatively co-prime, the conventional residue system where an integer is represented by its residues is non-weighted. Different from commonly used weighted system such as binary representation, the reconstruction error is not proportional to the magnitude of the residue error in general. From a perspective of geometry, when a group of moduli are relatively co-prime, i.e., the least common multiple equals the product of them, each integer can be one-to-one mapped to lattice points in a coordinate space, illustrated in \cite{TSP2018} \cite{towards}, determined by its resides. Taking it as a encoding system, the codes are on lines in direction $(1,1,...,1)$ within a cube. For convenience, we call such lines as distribution lines.  Even merely considering the lattice points, the minimal infinite distance between any two of them is $1$, implying the maximum error control capacity is $0.5$. Moreover, such error bound is also tight. Once the error exceeds the bound,  it will be closer to another lattice rather than the original one since the mapping from integers to the lattice points is bijection. When we take all real numbers within $[0, lcm(m_1, ... ,m_L))$ into account, the minimal distance of the codes in our setting is 0 and therefore it would be impossible for perfect error correction, different from the classic coding theory with errors measured in Hamming weight. To achieve the best error control performance, as proved in \cite{towards}, it is equivalent to figuring out which line the code of $Y_i$ is on, and the minimal distance among those distribution lines given moduli has been explored in \cite{TSP2018}. A corollary from \cite{TSP2018} is that when all moduli are increased by a common factor, $\Gamma$, the minimal distance will also increase by $\Gamma$ correspondingly. 

Apart from the minimal distance, to overcome the absence of the correspondence of residues, additional redundancy beyond introducing a common divider is necessary. So far, such distinguishability is mainly derived and proved by number theory foundation of classic remainder codes. The best result known we have is that: the lcm of all moduli should be in the order of $\prod_{i=1}^N Y_i$ and the error bound $\delta = \max_{il} |\Delta_{il}| < \frac{\Gamma}{4N}$. Such error bound can be generalized to  
\begin{equation}
\label{pre}
\max_{l} \Delta_{il} - \min_{l} \Delta_{il}<\frac{\Gamma}{2N}.
\end{equation} 

Here $\Delta_{il}$ denotes the error in the residue of $Y_i$ modulo $m_l$. Since the errors $\Delta_{il}$ introduced are independent, the probability that assumption (\ref{pre}) holds is  
\begin{equation}
\label{minprodensity}
{(\int_{-\infty}^{\infty} p( \min_{l} \Delta_{il} =x) \Pr(  \Delta_{il} \in [x,x+2\delta), l=1,2,...,L ) ~ dx)}^{N}
\end{equation}
where $p$ is the probability density function of $\min_{l} \Delta_{il}$. For a Gaussian noise, (\ref{minprodensity}) can be further upper bounded by 
\begin{equation}
\label{pre-upper}
{\Pr(  \Delta_{il} \in [-\delta, \delta) )}^{N(L-1)}
\end{equation}
Here for simplicity, we assume that the variance $\sigma^2_l$ is the same for each $l$. As $N$ increases, with fixed SNR, (\ref{pre-upper}) decays exponentially of $O(N^2)$. \footnote{$L$ indeed can be regarded as a linear function of $N$ where the fraction of $\frac{L}{N}$ is around the average number that the lcm of $\frac{L}{N}$ many moduli is bigger than $X_i$. } To compensate the loss, we have to increase $\Gamma$ sharply, which motivates us to consider other efficient estimators, where each time the number of samples does not depend on $N$. Based on CRT, the least number of sampling $L_{min}$ should satisfy $lcm (m_1, m_2, ... ,m_{L_{min}}) > \max_i Y_i$, in which case the maximal possible dynamic range of $Y_i$ is also achieved. As we state above, to meet the bound derived from number theory, for each time estimation, much more than $L_{min}$ times sampling are required while the more samples, the lower success rate. Furthermore, in our case to deal with the estimation under modulo functions, once the condition (\ref{pre}) is broken, the reconstruction error will be incredibly large. \footnote{Since the difference between any two numbers even from two adjacent distribution lines can be as large as in a scale of $D$ due to the non-weighted property of residue system. Thereby when a wrong distribution line found, the estimation becomes meaningless. More details can be referred to the performance simulation in \cite{closed} \cite{RCRT-two}. } Therefore, we try classifying the residues by resorting to statistics with least samplings. 

It is noted that 
\begin{equation}
 \mu_{i}  = \langle \langle Y_i \rangle_{m_l} \rangle_{\Gamma} = \langle Y_i \rangle_{\Gamma} 
\end{equation}
As a property shared by all residues of $Y_i$ modulo $m_l$, we call $\mu_{i} $ the common residue of $Y_i$. When $\mu_{i} $ are distinct, such syndrome can be used to distinguish the correspondence of residues. However, with combined occurrence of errors, such detection is not strong enough but it inspires us to consider the underlying relationship between $\mu_{i}$ and the MAP of the residue classification. Throughout the rest of the paper, we will show that on achieving the maximal possible dynamic range, all the statistical analysis on $Y_i$ can be elegantly reduced to that on those erroneous common residues $r_{il}= \langle R_{il} \rangle_{\Gamma}$. Here $R_{il}$ denotes the erroneous residue of $Y_i$ modulo $m_l$ observed, i.e., $\langle Y_i + \Delta_{il} \rangle_{m_{l}}$. For the convenience of readers, all constantly used notations are listed in the above Table I.

\begin{table}
\caption{List of Notations}
\begin{tabular*}{8.8cm}{lll}
\hline
Notations & ~~~~~~~~~ Explanation    \\
\hline
$L$  & The number of samplings / moduli selected  \\
$ m_l$ & Moduli selected  \\
$\mathscr{M}$ & $\mathscr{M}=\{m_l, l=1,2,...,L\}$ \\
$N$ & The number of real numbers to be reconstructed\\
$Y_i$ & Real number to be reconstructed\\
$K_l$ &Permutation variable for each sampling\\
$R_{il}$ & Raw observations $i,l$ \\
$\delta_{il}$ & random noise in observation $i,j$\\
$\mu_i$ & Common residue, residue of real number $Y_i$ modulus by $\Gamma$\\
$r_{il}$& Residue of observation $R_{ij}$ modulus by $\Gamma$ \\

$\hat{Y}_{i}$ & Estimation of $Y_i$ \\
$\hat{K}_l$ &Estimation for Permutation variable of each sampling\\
$\hat{\mu}_{i}$ & Estimation for residue of real number $Y_i$ modulus by $\Gamma$\\

\hline
\end{tabular*}
\end{table}

\section{Algorithm One: Deterministic Maximum a Posteriori Estimation For Clustering}
\noindent In this section, we introduce our non-informative prior and restate our target problem as a Bayesian statistical model. We further show that under Assumption 1, the MAP of residue classification is in a semi-closed form and can be determined from $O(NL)$ candidates. Relying on the MAP of classification, it is reduced to $N$ independent conventional RCRT for a single number. 



We start from introducing our target model as a Bayesian setup. Let $N$ real numbers $Y_1, Y_2, ..., Y_N$ be our target parameters to reconstruct, and we assign their prior as uniformly distributed under domain $[0,D)$. In total, $L$ times sampling are implemented, and a set of moduli $m_l = \Gamma \times M_l, l = 1, 2, ..., L$ and $\sigma_l, l = 1, 2, ..., L$, are designed and preset. As mentioned in above sections, $\Gamma$ is a preset real number, and $M_l$ are relatively co-prime. On achieving the maximal dynamic range, according to CRT, $D$ is set as $D = \Gamma \times \prod_{l = 1}^{L} M_l$. 

During the $l^{th}$ sampling, we will observe an unordered set of noisy residues, denoted by $\bm{R}_{[1:N],l}=(R_{1l}, R_{2l}, ... , R_{Nl})$ and $\bm{R}_{[1:L]}=(\bm{R}_{[1:N],1}, \bm{R}_{[1:N],2}, ... , \bm{R}_{[1:N],L})$. We know they are the residues of $Y_{[1:N]}$ modulo $m_l$, while we do not know their correspondence relationships. To specify the problem, we introduce $\textbf{K}_l, l =1, 2, .., L$, as a set of i.i.d. $N$-permutation variable, so that we may specify that $R_{{K_{l}(i)},l} = \langle Y_i + \Delta_{il}\rangle_{m_l}, i = 1, 2,..., N, l = 1, 2,..., L$, where $\Delta_{il}, i= 1,2,.., N$ are i.i.d random noise following gaussian distribution $N(0, \sigma_i^2)$. A non-informative prior, uniform distribution, is also assumed for the permutation variable $\textbf{K}_{[1:L]}=({K}_1, {K}_2, ... ,{K}_L)$. 

To simplify future discussions on the 'smaller circle' (modulo $\Gamma$), we decompose $Y_i$ into $Y_i = k_i \Gamma + \mu_i$, where $\mu_i := \langle Y_i\rangle_{\Gamma}$ denotes the residue of our target parameter $Y_i$ modulo $\Gamma$, and $k_i$ denotes the corresponding quotient. Since $Y_i$ follows a prior of uniformed distribution under $[0, D)$, we assign $k_i$ as an integer random variable uniformly distributed within $\{0, 1, 2,..., \frac{D}{\Gamma}-1\}$, and $\mu_i$ uniformly distributed within $[0,\Gamma)$. Similarly, we decompose $R_{il}$ into $j_{il}  \Gamma +  r_{il}$, where $r_{il} := \langle R_{ij} \rangle _{\Gamma}$ denote residues modulo by $\Gamma$, and $j_{il}$ denotes the quotient. We therefore move all parameters and observations onto a "smaller circle", with modulo $\Gamma$.  

Considering nature of the target problem, we set the permutation variable $\textbf{K}_{[1:L]}$ as our target variable in this algorithm, and design the objective as a Maximum a Posteriori Estimation (MAP) for $\textbf{K}_{[1:L]}$, denoted by $\hat{\textbf{K}}_{[1:L]}$ i.e.
\begin{equation}
\label{huahua}
\begin{aligned}
\hat{\textbf{K}} _{[1:L]}
& := arg_{\textbf{K}}\max  p(\textbf{K}_{[1:L]}|  \textbf{R}_{[1:L]}  )  \\
& \propto arg_{\textbf{K}}\max p(\textbf{R}_{[1:L]}  |\textbf{K}_{[1:L]}) \\ 
& \propto arg_{\textbf{K}}\max \int_{Y_1} ... \int _{Y_N} p(\textbf{R}_{[1:L]}| \textbf{Y}, \textbf{K}_{[1:L]} ) dY_1...dY_N \\
\end{aligned}
\end{equation}
where $\textbf{Y}$ is similarly the compact form of $(Y_1, Y_2, ... ,Y_N)$. The complexity of directly solving the objective function mainly comes from the comparison among a total $L \times N!$ scenarios for different $\textbf{K}_{[1:L]}$ as well as the vagueness in measuring the Gaussian noise $\Delta_{il}$. The exponential number of scenarios for permutation is pretty obvious. We now further discuss on the second problem. 

When conditioning on $\textbf{K}_{[1:L]} $, we would know the correspondence between $\textbf{Y}$ and $\textbf{R}_{[1:L]}$. Therefore, the integration in equation (\ref{huahua}) can be reduced to calculating the following equation, where we remove $\textbf{K}$ and assume $R_{il}$ corresponds to $Y_i$:

\begin{equation}
\label{deduction2}
\begin{aligned}
&\int_{Y_i} p(\textbf{R}_{[1:L]}| \textbf{K}_{[1:L]} ,Y_i ) dY_i \\
&\propto \int_{0}^{\Gamma} \sum_{k_i = 0}^{\frac{D}{\Gamma}} \prod_{l = 1}^{L} \sum_{j_{il} = -\infty}^{\infty} p(j_{il}  \Gamma +  r_{il} | k_i \Gamma + \mu_i ) d\mu_i  \\
&\propto \int_{0}^{\Gamma} \sum_{k_i = 0}^{\frac{D}{\Gamma}} \prod_{l = 1}^{L}  \sum_{j_{il} = -\infty}^{\infty} \frac{1}{\sqrt{2\pi}\sigma_l} e^{\frac{-(r_{il} - \mu_i + (j_{il} - k_i )\Gamma)^2}{2\sigma_l^2}}    d\mu_i \\
 &\propto \int_{0}^{\Gamma}   \prod_{l = 1}^{L} \sum_{j'_{il} = -\infty}^{\infty}\frac{1}{\sqrt{2\pi}\sigma_l} e^{\frac{-(r_{il} - \mu_i + j'_{ij}\Gamma)^2}{2\sigma_l^2}}    d\mu_i 
\end{aligned}
\end{equation}

Since $j_{il}$ are inherently independent among different observations, it's hard to deduce further. However, it is noted that when $L = 1$, i.e., only one data point was observed, we would successfully integrate out $\mu_i$ as shown in following equation: 
\begin{equation}
\begin{aligned}
 &\sum_{j'_{il} = -\infty}^{\infty} \int_{0}^{\Gamma} \frac{1}{\sqrt{2\pi}\sigma_l} e^{\frac{-2(r_{il} - \mu_i + j'_{il}\Gamma)^2}{2\sigma_l^2}}    d\mu_i \\
&= \int_{-\infty}^{\infty} \frac{1}{\sqrt{2\pi}\sigma_l} e^{\frac{-2(r_{il} - \mu_i)^2}{2\sigma_l^2}}    d\mu_i
\end{aligned}
\end{equation}

In the following, we introduce Assumption 1, under which a polynomially fast algorithm is creatively proposed to deterministically derive the MAP estimation for $\textbf{K}_{[1:L]}$. We start from introducing some notations for noise intervals: for each $i = 1, 2,..., N$, we define an interval $I_i$ as $I_i = [ \mu_i+ \min_{l} \Delta_{il},  \mu_i+ \max_{l} \Delta_{il} ]$, i.e., starting from $\mu_i+ \min_{l} \Delta_{il}$ clockwise to $\mu_i+ \max_{l} \Delta_{il}$. Such notation is illustrated in Fig. \ref{circle-interval}. In addition, let $\Omega_{i}$ denote the length of the directed interval $I_i$, i.e.. $\Omega_{i}:=\max_{l} \Delta_{il} - \min_{l} \Delta_{il}$. 

\begin{figure}
\centering
\includegraphics[width=2.87 in,height=1.66 in]{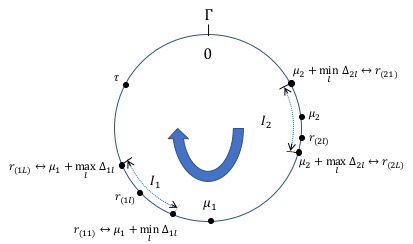}
\caption{Illustration for the noise interval}
\label{circle-interval}
\end{figure}

\begin{cond}  
\label{cond}
There exists some point $\tau$ on the circle modulo $\Gamma$ such that it is not within any interval $I_i$ and for $i=1,2,...,N$, $\Omega_i < \frac{\Gamma}{2}$.
\end{cond}

As illustrated in Fig.  \ref{circle-interval}, $\tau$ can be arbitrary any point on the circle which is not overlapped by any $I_i$. From the principle of clustering, still we would like to figure out the least requirement of $\Delta_{il}$ to achieve robustness that even a perfect classification exists. As explained before, under perfect classification, our problem will be reduced to $N$ independent RCRT for a single number. In \cite{closed}, \cite{sp2017}, such robustness is proved to be achieved when $\Omega_i < \frac{\Gamma}{2}$ for $i=1,2,...,N$. A generalized version with MLE techniques is proposed in \cite{mle1}. For simplicity, we just assume that $\Omega_i < \frac{\Gamma}{2}$ for each $i$. Indeed the following results can also be generalized under the loosened condition in \cite{mle1}. Also, it is not hard to observe that if $\sum_{i=1}^{N} \Omega_{i} < \Gamma$, then such $\tau$ must exist.

Assumption 1 provides convenience in distance measurement where the circle can be virtually cut at point $\tau$, and stretch into a line. From the relative positions of $r_{il}$ on the line, we can deduce an 'ordered relationship' among points. Specifically, we denote $\{r_{(il)}, l = 1, 2,..., L\}$ as an clockwise ordered sequential statistics of $\{r_{il}, l = 1, 2,..., L\}$. Here, $(il)$ denotes a permutation on the index product $\{il, l = 1,2,...,L\}$ for each $i$, such that $r_{(il)}$ is ordered $l^{th}$ among $r_{i1}, r_{i2},...,r_{iL}$ clockwise starting from $\tau$, illustrated in Fig. \ref{circle-interval}. 

\begin{lm} 
\label{cutting}
For each $i$, the errors $\Delta_{il}$ of the subsequence $r_{(i1)},r_{(i2)}, ...,r_{(iL)}$ are in an ascending order accordingly. Moreover, following the order, we can figure out the ordered statistics of $\Delta_{il}$ for each $i$.
\end{lm}

\begin{proof}
Since $\tau$ is not overlapped by any $I_i$, a necessary equation can be derived from Assumption 1 that $\Omega_{i} < \Gamma$ for any $i$. Therefore the directed interval defined is clockwise distributed starting from $\langle \mu_i + \min_l \Delta_{il} \rangle_{\Gamma}$ to $\langle \mu_i + \max_l \Delta_{il} \rangle_{\Gamma}$ and $\tau$ is in the complementary part, $[0,\Gamma) / I_i$. Therefore, clockwise, the closest one among $r_{i1}, r_{i2},...,r_{iL}$ to $\tau$ should be $r_{(i1)}$, and counterclockwise the closest one is $r_{(iL)}$. Therefore, for each $i$, $r_{(il)}$ is exactly in the order where correspondingly $\Delta_{il}$ are arranged ascendingly. 
\end{proof}

To reflect such relative positions of $r_{il}$ under the permutation $(il)$, we add a shift on $r_{i1},r_{i2}, ...,r_{iL}$ as follows. Essentially, we convert them to the expression on an axis rather than a circle.

\begin{df}
When $0 \leq \tau \leq \min {r}_{{(il)}}$ or $\max {r}_{{(il)}} \leq \tau <\Gamma$, for $i=1,2,...,N$ and $l=1,2,...,L$,

\begin{equation}
\label{shift}
\tilde{r}_{{(il)}} = {r}_{{(il)}}  
\end{equation}

Otherwise,
\begin{equation}
\begin{cases}
\tilde{r}_{{(il)}} = {r}_{{(il)}},  ~~~~~when~~ {r}_{{(il)}} \leq \tau \\
\tilde{r}_{{(il)}} = {r}_{{(il)}}-\Gamma, ~~~~ when~~ {r}_{{(il)}} > \tau  \\
\end{cases}
\end{equation}
\end{df}




Under Assumption 1, we have to further define \textit{proper classification}, where the specific statements are postponed but the following lemma provides the motivation and some insights. 
\begin{lm}
Given $r_{il}$ and $K_l$ for $i=1,2,...,N$ and $l=1,2,...,L$, under the assumption that $\Omega_i < \frac{\Gamma}{2}$, $I_i$ can be determined.
\end{lm}

\begin{proof}
With $K_l$, all samples $r_{il}$ are divided into $N$ subsets and assume that $\{r_{K_l(i)l}, l=1,2,...,L\}$ are all the samples for $\mu_i$. When Assumption 1 holds, there should exist an interval clockwise over the circle starting from $r_{K_{l_1}(i)l_1}$ while ending at $r_{K_{l_2}(i)l_2}$ for some $l_1 \not = l_2 \in \{1,2,...,L\}$ such that the length is smaller than  $\frac{\Gamma}{2}$. All the rest samples $\{r_{K_l(i)l}, l=1,2,...,L\}/ \{r_{K_{l_1}(i)l_1},r_{K_{l_2}(i)l_2}\}$ lie in the interval. Clearly, $r_{K_{l_2}(i)l_2}$ and $r_{K_{l_1}(i)l_1}$ are clockwise neighboring with $\{r_{K_l(i)l}, l=1,2,...,L\}$ in between. We claim such interval is unique and is exactly $I_i$. Otherwise, let there be indices $l_3$ and $l_4$ such that the clockwise-directed interval starting from $r_{K_{l_3}(i)l_3}$ to $r_{K_{l_4}(i)l_4}$ is also in a length smaller than $\frac{\Gamma}{2}$ and contains $\{r_{K_l(i)l}, l=1,2,...,L\}$. Let such interval be denoted by $I'_1$, then $|I'_1|$ equals $\Gamma$ minus the counter part of $I'_1$, i.e, the part starting from $r_{K_{l_4}(i)l_4}$ to $r_{K_{l_3}(i)l_3}$, which is included in $I_1$. Because $|I_1|$ is smaller than $\frac{\Gamma}{2}$. Therefore $|I'_1| \geq \Gamma - |I_1| >\frac{\Gamma}{2}$, a contradiction. Thus, our claim holds.
\end{proof}

To proceed from Lemma 2, a  \textit{proper classification} is that for a possible classification $K_1,K_2,...,K_L$, given $r_{il}$, Assumption 1 still holds, i.e., under  such classification, we can find out the unique $I_i$ where $|I_i|<\frac{\Gamma}{2}$ and there exists a $\tau$ which is not overlapped by any $I_i$. If a classification is not $proper$, then $\Pr(K_1, K_2, ... ,K_L, r_{il}, Assum ~1)=0$. In the following, for simplicity, $Assum ~1$ will stand for Assumption 1, shown in the formulas. 

We further modify our objective function as following: 
\begin{align}
\begin{split}
\hat{K} _{[1:L]}
& := arg_{\textbf{K}}\max  p(\textbf{K} |  \textbf{R}_{[1:L]}  , Assum ~1)  \\
& \propto arg_{\textbf{K}}\max p(\textbf{R}_{[1:L]}  | Assum ~1, \textbf{K}) \times p(Assum ~1, \textbf{K}) \\ 
\end{split}
\end{align}

It's pretty obvious that $Assumption ~ 1$ is independent of permutation $\textbf{K}$. For any $\textbf{K}_0$, 
\begin{align}
\begin{split}
\Pr (Assum~1 | \textbf{K}_0) = \int_{ \textbf{R}_{[1:L]}} p(Assum ~ 1, \textbf{R}_{[1:L]}| \textbf{K}_0) d \bm{R}_{[1:L]} \\
= \int_{\bm{R}_{[1:L]}} p(   Assum ~ 1 |\textbf{K}_0  , \textbf{R}  ) \times p(\textbf{R} | \textbf{K}_0 )d \bm{R}_{[1:L]} \\ 
= \int_{\bm{R}_{[1:L]}} p(  Assum ~ 1 |\textbf{K}_0 , \bm{R}_{[1:L]}  ) \times p(\bm{R}_{[1:L]} ) d \bm{R}_{[1:L]}
\end{split}
\end{align}
Also, 

\begin{align}
\begin{split}
\Pr (Assum~1) =\int_{\bm{R}_{[1:L]}} \sum_{\textbf{K}} p(Assum~1, \textbf{K}, \bm{R}_{[1:L]}) \times p(\textbf{K}) d\bm{R}_{[1:L]}\\
=\sum_{\textbf{K}} \int_{\bm{R}_{[1:L]}} p(Assum ~1|\textbf{K},\bm{R}_{[1:L]}) \times p(\bm{R}_{[1:L]}|\textbf{K}) \times p(\textbf{K}) d\bm{R}_{[1:L]}
\end{split}
\end{align}

Since we set prior information for $\textbf{K}$ as uniformly distributed, showing independence suffices to prove $\int_{R} p(Assum~1|\bm{K},\bm{R}_{[1:L]}) \times p(\bm{R}_{[1:L]}|\bm{K}) d\bm{R}_{[1:L]}$ remain constant across all \textbf{K}. On the other hand, as $\bm{R}_{[1:L]}|\textbf{K}$ is a normal distribution, we know $\int_{\bm{R}_{[1:L]}} p(Assum~1|\textbf{K},\bm{R}_{[1:L]}) \times p(\bm{R}_{[1:L]}|\textbf{K}) d\bm{R}_{[1:L]}$ is constant across all $\textbf{K}$. Thus our claim follows.

 If $\{\textbf{K}_{[1:L]}\}$ is a $proper~classification$, assuming that the $cutting~point$ is $\tau$, following the notations given in Definition 1, then a closed form of (\ref{huahua}) can be obtained as follows,
 
\begin{lm} When $\Pr(\textbf{K}_{[1:L]},\textbf{r}_{[1:L]}, Assum ~1) \not =0$, 
 \begin{equation}
 \begin{aligned}
 \label{MAP}
 \Pr(\textbf{R}_{[1:L]} | Assum~1,& \textbf{K}_{[1:L]} )  = \Pr(\textbf{r}_{[1:L]} | Assum~1, \textbf{K}_{[1:L]} )  =  \\
& \prod_{i=1}^{N} \int_{-\infty}^{\infty} e^{-\sum_{l=1}^{L} w_l (x-\tilde{r}_{K_l(i)l})^2} dx
  \end{aligned}
 \end{equation}
 \end{lm}
 where $w_l$ is the weight determined by $\sigma_l$.

 \begin{proof}
 First we need to clarify when there are multiple cutting point candidates given $r_{il}$ and $K_l$, (\ref{MAP}) is invariant with any selection of $\tau$. For each $i$, since for any possible $\tau \not \in I_i$, the relative position of $\tilde{r}_{K_l(i)l}$ correspondingly defined does not change, which is proved in Lemma 1. The only difference is that there may be a uniform shift on $\tilde{r}_{K_l(i)l}$, i.e., different cutting points may result in two different groups $\{\tilde{r}_{K_l(i)l}\}$ and $\{\tilde{r'}_{K_l(i)l}\}$ but $|\{\tilde{r}_{K_l(i)l}\}-\{\tilde{r'}_{K_l(i)l}\}|=\Gamma$ for each $l$. Substitute both $\{\tilde{r}_{K_l(i)l}\}$ and $\{\tilde{r'}_{K_l(i)l}\}$ to (\ref{MAP}), the formula does not change due to the integral on the variable along the whole real axis.  
 
 For each $i$ and a fixed $\mu_i \in [0,\Gamma)$, with the assumption of $I_i$, the order of $l$ such that $\mu_i+\Delta_{k_l(i)l}$ are in an ascending order equals that of $\tilde{r}_{K_l(i)l}$ sorted non-decreasingly. Therefore, the errors $\{ \Delta_{K_l(i)l}, l=1,2,...,L\}$ must be in a form $\{\mu_i-\tilde{r}_{K_l(i)1}+j\Gamma,\mu_i-\tilde{r}_{K_l(i)2}+j\Gamma, ... ,\mu_i-\tilde{r}_{K_l(i)L}+j\Gamma\}$, where $j \in \mathbb{Z}$. On the other hand, since $\mu_i$ are i.i.d. uniformly distributed in $[0,\Gamma)$, putting things together, for each $i$, under Assumption 1 and classification, the probability of such generated samples $r_{il}$ is $\int_{-\infty}^{\infty} e^{-\sum_{l=1}^{L} w_l (x-\tilde{r}_{K_l(i)l})^2} dx$. Due to the independence of $\mu_i$, (\ref{MAP}) follows. 
 
 \end{proof}

For any possible \textit{proper~classification}, with Lemma 3, we have a closed form of  (\ref{huahua}). Indeed, there exists an efficient scheme to find out the optimal solutions of  (\ref{huahua}).  Before proceeding, we introduce the following notations for clarity. For any given $\tau \in [0,\Gamma)$, let $\gamma_{(i)l}$ denote the $i^{th}$ order of $\tilde{r}_{il}$, for each $l \in \{1,2,...,L\}$ sorted in an ascending order.
 \begin{thm}
 The MAP estimation of clustering under Assumption 1 is to find out a cutting point $\tau \in \{r_{il}\}$ such that 
 \begin{equation}
 \label{thm2}
 \arg \min_{\tau}  \sum_{i=1}^{N}  (\sum_{l=1}^{L} \gamma_{(i)l} w_l)^2
 \end{equation}
 \end{thm}

\begin{proof}
For any \textit{proper~classification} $\textbf{K}_{[1:L]} = (K_1,K_2,...,K_L)$, (\ref{MAP}) can be further simplified as,
\begin{equation}
\begin{aligned}
\label{combine}
 & \prod_{i=1}^{N} \int_{-\infty}^{\infty} e^{-[ (\sum_{l=1}^{L} w_l) x^2 -2\sum_{l=1}^{L}\tilde{r}_{K_l(i)l} w_lx+ \sum_{l=1}^{L} w_l \tilde{r}^2_{K_l(i)l}]} dx\\
&=\prod_{i=1}^{N} \int_{-\infty}^{\infty} exp [- (\sum_{l=1}^{L} w_l x-\frac{\sum_{l=1}^{L}\tilde{r}_{K_l(i)l} w_l}{\sum_{l=1}^{L} w_l})^2  \\
&- (\frac{\sum_{l=1}^{L}\tilde{r}_{K_l(i)l} w_l}{\sum_{l=1}^{L} w_l})^2+ \sum_{l=1}^{L} w_l \tilde{r}^2_{K_l(i)l}]  dx\\
& \propto \sum_{i=1}^{N} [- (\frac{\sum_{l=1}^{L}\tilde{r}_{K_l(i)l} w_l}{\sum_{l=1}^{L} w_l})^2+ \sum_{l=1}^{L} w_l \tilde{r}^2_{K_l(i)l}]\\
\end{aligned}
\end{equation}

Given $r_{il}$, for any $\textbf{K}_{[1:L]} = (K_1,K_2,...,K_L)$ which can result in a same $\tau$, they will share the same $\tilde{r}_{il}$ and therefore in  (\ref{combine}) $\sum_{i=1}^{N}\sum_{l=1}^{L} w_l \tilde{r}^2_{K_l(i)l}$ is a constant and we only need to focus on 
\begin{equation}
\label{final}
 \sum_{i=1}^{N}  (\sum_{l=1}^{L}\tilde{r}_{K_l(i)l} w_l)^2
\end{equation}
For the rest, we first prove that the clustering following the rule that grouping $\{\gamma_{(i)1},\gamma_{(i)1},...,\gamma_{(i)1}\}$ for each $i$ achieves the maximum of (\ref{final}). This is a generalization of the following inequality. For two pairs of numbers $a_1 \leq a_2$ and $b_1 \leq b_2$, then 
\begin{equation}
{(a_1+b_1)}^2+{(a_2+b_2)}^2 \geq {(a_1+b_2)}^2+{(a_2+b_1)}^2
\end{equation}
Now in general, for $2$ sequences, each of $N$ numbers: $\{\gamma_{(1)1},\gamma_{(2)1},...,\gamma_{(N)1}\}$ and $\{\gamma_{(1)2},\gamma_{(2)2},...,\gamma_{(N)2}\}$ and both sorted in a non-decreasing order, respectively, then the Rearrangement Inequality \cite{inequalities} tells that 
\begin{equation}
\label{inequ}
\begin{aligned}
\gamma_{(K(1))1}\gamma_{(1)2} &+ \gamma_{(K(2))1}\gamma_{(2)2}+...\gamma_{(K(N))1}\gamma_{(N)2}  \\
& \leq \gamma_{(1)1}\gamma_{(1)2} + \gamma_{(2)1}\gamma_{(2)2}+...\gamma_{(N)1}\gamma_{(N)2}
\end{aligned}
\end{equation}
where $K$ can be any permutation on $\{1,2,...,N\}$. Said another way, the maximum achieves when the order is preserved. With (\ref{inequ}), the optimal value of (\ref{final}) is clear:  $\sum_{i=1}^{N}  (\sum_{l=1}^{L} \gamma_{(i)l} w_l)^2$, as we claimed, of which the optimal classification is obviously a $proper~classification$ with the cutting point $\tau$ assumed. 

Since there are $NL$ many candidates of \textit{cutting~point}, and what we prove above presents the local optimal classification for all \textit{proper~classification} which shares a same $\tau$. Therefore in the worst case, by enumerating all the $NL$ candidates of \textit{cutting~point}, i.e., each $r_{il}$, we can find the final optimal solution to (\ref{huahua}). 
\end{proof}
To conclude, the complexity of finding out the MAP for clustering under Assumption 1 is reduced to finding out the optimal $\tau$ from $NL$ many candidates $r_{il}$.  We conclude the proposed algorithm as follows.

\begin{algorithm}
\caption{Conditional MAP Estimation of Classification}
\textbf{Input:} Given moduli $m_l=\Gamma M_l$ and the residues observed $R_{il}, i=1,2,...,N, l=1,2,...,L$. 

1. Calculate $ {r}_{il} = \langle  R_{il} \rangle_{M_l}$;

2. Begin iteration $t$ from 1 to $NL$ for each $ {r}_{il}$. For notation clarity, let $  {r}_{i_0l_0}$ denote the residue selected in an iteration.
    \begin{itemize}
    \item For each $i$ and $l$, if ${r}_{il} > {r}_{i_0l_0}$, 
    \begin{equation}
    \widetilde{r}_{il} = {r}_{il}  - \Gamma.
    \end{equation}
    Otherwise $\widetilde{r}_{il} = {r}_{il} $.
    \item Let $\widetilde{r}_{(i)l}$ denote that it is the $i^{th}$ largest elements among all $\{\widetilde{r}_{il},i=1,2,...,N\}$. Derive the permutation $K_l$ that $ K_l(i) = (i), l=1,2,...,L$, and calculate (\ref{final}) under the classification.
       \end{itemize}
        
3. Find out the $\tau$ such that (\ref{thm2}) achieves the minimal. Applying the conventional RCRT for a single number \cite{mle1}, \cite{closed} for the residues in $C^{\tau}_{i}$, and $\hat{Y}_i$ is thus consecutively obtained. 

\textbf{Output:} $\hat{Y}_i$, $i=1,2,...,N$. 
\end{algorithm}

\section{Algorithm Two: Bayesian Wrapped Gaussian Mixture Model and Two-Step Maximization Fast Algorithm}
\noindent In the last section, we discussed a conditional MAP of residue classification. It is noted that after permutation estimators are determined, for final reconstruction, we still need to use conventional RCRT. Conventional RCRT relies its key estimation on $\bm{\mu}_{[1:N]}=(\mu_{1},\mu_{2},...,\mu_{N})$, an random variable we integrate out during our estimation for permutation variable $\bm{K}$. It is therefore inspiring that whether we can estimate both permutation $\textbf{K}_{[1:L]}$ and common residue $\bm{\mu}_{[1:N]}$ at the same time.

In this section, we develop a two-step searching algorithm to figure out the defined MAP. Coming with a slight compromise in computational complexity, Algorithm 2 is shown to achieve stronger robustness comparing to Algorithm 1.

To begin, if we further place the problem onto the "small circle" modulo $\Gamma$, we would find its similarity to the Gaussian Mixture Model (GMM), where the differences lie on the wrapped gaussian distribution for noisy observations of $\textbf{Y}_{[1:N]}$, and instead of a clustering problem, our missing data permutation variable $\bm{K}$ makes it a matching problem. 

In this session, we will treat both $\bm{\mu}_{[1:N]}$  and $\textbf{K}_{[1:L]}$ as target variables, and conduct a MAP estimation for both variables at same time, whose objective function becomes: 

\begin{equation}
\begin{aligned}
&\bm{\hat{K}}_{[1:L]}, \bm{\hat{\mu}}_{[1:N]}\\
& :=arg_{\bm{K}_{[1:N]}, \bm{\mu}_{[1:N]}}\max  \Pr (\bm{K}_{[1:N]}, \bm{\mu}_{[1:N]} |  \textbf{r}_{[1:L]}  )  \\
&\propto arg_{\bm{K}_{[1:N]}, \bm{\mu}_{[1:N]}}\max \Pr (\textbf{r}_{[1:L]}  |  \bm{K}_{[1:N]}, \bm{\mu}_{[1:N]}) \\ 
\end{aligned}
\end{equation}

We propose an iterative algorithm to solve the above problem. It goes as following: after initializing a certain $\bm{\mu}_{[1:N]}^{(0)}$, for $(t+1)^{th}$ iteration, the maximization is broken into two steps: 
\begin{itemize}
	\item Step One: knowing $\bm{\mu}_{[1:N]}^{(t)}$, deducing:
	\begin{equation}
	\label{StepOne}
	\bm{\hat{K}}_{[1:L]}^{(t+1)} = arg_{\bm{K}_{[1:N]}}\max \Pr (\textbf{r}_{[1:L]}  | \bm{K}_{[1:N]}, \bm{\mu}_{[1:N]}^{(t+1)})
	\end{equation}
	\item Step Two: knowing $\bm{K}_{[1:N]}^{(t +1)}$, deducing:
	\begin{equation}
	\bm{\hat{\mu}}_{[1:L]}^{(t+1)} = arg_{\bm{\mu}_{[1:N]}}\max \Pr (\textbf{r}_{[1:L]}  | \bm{K}_{[1:N]}^{(t+1)}, \bm{\mu}_{[1:N]})
	\end{equation}
\end{itemize}

In the remaining part of this section, we will propose a fast algorithm to deal with each step and prove its convergence to local minimum. In Section 5, we will provide simulation results for this algorithm, where the convergence speed as well as simulation accuracy will be further demonstrated. 

Worth mentioned, in algorithm and simulations below, we will always initialize $\bm{\mu}_{[1:N]}^{(0)}$ as the $l^{th}$ set of observation $\bm{r}_{[1:N],l} = \{r_{1,l}, r_{2,l},...,r{N,l}\}$, where $l$ is a randomly drawn from $\{1,2,...,L\}$. As our missing data has its nature as permutation, this initialization is relatively good guess. 

We start from deducing a fast algorithm for Step One. Similar to equation (\ref{deduction2}), we have 

\begin{equation}
\label{target two}
\begin{aligned}
&\Pr(\textbf{r}_{[1:L]}| \bm{K}_{[1:L]}, \bm{\mu}_{[1:N]} ) \\
& \propto \prod_{l = 1}^{L} \prod_{i = 1}^{N} \sum_{j_{il} = -\infty}^{\infty} p(j_{il}  \Gamma +  r_{il} | k_i \Gamma + \mu_{K_{l}}(i)  )  \\
&\propto \prod_{l = 1}^{L} \prod_{i = 1}^{N} \sum_{j_{il} = -\infty}^{\infty} \frac{1}{\sqrt{2\pi}\sigma_l} e^{\frac{-(r_{il} - \mu_{K_{l}}(i)  + (j_{il} - k_i )\Gamma)^2}{2\sigma_l^2}} \\
 &\propto \prod_{l = 1}^{L} \prod_{i = 1}^{N} \sum_{j'_{il} = -\infty}^{\infty}\frac{1}{\sqrt{2\pi}\sigma_l} e^{\frac{-(r_{il} - \mu_{K_{l}}(i)  + j'_{ij}\Gamma)^2}{2\sigma_l^2}} 
 \end{aligned}
 \end{equation}

Since $K$ are independently drawn random permutation, we may simplify the problem in (\ref{target two}) into solving an optimal $K_l^{(t+1)}$ for each $l$:

\begin{equation}
\label{target two further}
\begin{aligned}
&K_l^{(t+1)} := \arg_{K_l}\max \prod_{i = 1}^{N}\sum_{j'_{il} = -\infty}^{\infty} e^{\frac{-(r_{il} - \mu_{K_{l}}(i) + j'_{ij}\Gamma)^2}{2\sigma_l^2}} 
\end{aligned}
\end{equation}

In general, (\ref{target two further}) is hard to optimize, we therefore introduce a similar approximation as in \cite{mle1}. We define $d_{\Gamma} (a,b) := \min_{j} |a-b+j\Gamma |$ for any two real numbers $a,b$. When the variance is much smaller than $\Gamma$, (\ref{target two further}) can be approximated as 
\begin{equation}
\label{circle1}
\arg_{K_l} \min \sum_{i=1}^{N} d^2_{\Gamma} (r_{il}, \mu_{K_l(i)})
\end{equation}

Although still with seemingly exponential complexity, in the following theorem, we propose an algorithm that solves (\ref{circle1}) in $O(N)$ complexity. For convenience, we continue using the same notation as in Section 3, and let $(i)$, $i=1,2,...,N$, denote a permutation of sequence $r_{[1:N,l]}$ such that the sequence $\{ r_{(1)l}, r_{(2)l}, ... ,r_{(N)l} \}$ is increasingly sorted. We also further denote $[i]$ a permutation of sequence $\hat{\mu}_{[1:N]}$ such that the sequence $\{ \hat{\mu}_{[1]}, \hat{\mu}_{[2]}, ... ,\hat{\mu}_{[N]}\}$ is also non-decreasingly sorted in $[0,\Gamma)$. 

\begin{thm} 
	\label{circlemin1}
	There exist some $\zeta \in \{1,2,...,N\}$ such that the pairings $(\hat{\mu}_{[i]}, r_{(\langle i+\zeta\rangle_N ) l })$ achieves the optimal in (\ref{circle1}). 
\end{thm}

\begin{proof}
	Since each point over the circle can be uniquely represented by one argument, let $\theta_{i}, \omega_{i} \in [0,2\pi)$ respectively denote the arguments for $\hat{\mu}_{[i]}$ and $r_{(i)l}$. Clearly, $\theta_{i}$ and $\omega_{i}$ are both in an ascending order. Correspondingly, the distance between any $(\hat{\mu}_{[i_1]}, r_{(i_2)l})$ pair is proportional to $\min \{|\omega_{i_2}-\theta_{i_1}|, 2\pi-|\omega_{i_2} - \theta_{i_1}|\}$. Here we use $\xi_{i_1, i_2}$ to represent a directed argument associated with this distance, which is in $(-\pi,\pi]$. To give its geometrical interpretation, consider two concentric circles, as shown in Fig. \ref{concentric}, where $\hat{\mu}_{[i]}$ are distributed on the outer circle and $r_{(i)l}$ are on the inner. Here $(i)$ and $[i]$ denote a permutation on the index $\{1,2,...,N\}$ and they are ordered clockwise. Then $\xi_{i_1, i_2}$ is defined to represent a displacement from $\theta_{i_1}$ to $\omega_{i_2}$ taking the clockwise direction as positive. Imagine that there is a connection between each pairing of the optimal choice $(\hat{\mu}_{[i_1]}, r_{(i_2)l })$, we prove those $N$ lines have no intersection. 
	
	We prove it by contradiction. Suppose there exists $i_1$ and $i_2$ such that the two lines $\xi_{i_1, K(i_1)}$ and $\xi_{i_2, K(i_2)}$ are crossed. Without loss of generality, we take $\theta_{i_1}=0$, as there is an invariance of uniform shifting on the assignment of $\theta_{i}$ and $\omega_{i}$ with a zero point change. Also we take $\theta_{i_2}\in (0,\pi)$. Note that when $\theta_{i_2}=\pi$ it is impossible to result in crossing. So when two lines are crossed, it falls into one of the following four cases:
	\begin{itemize}
		\item 1). $\xi_{i_1, K(i_1)}\geq 0$ and $\xi_{i_2, K(i_2)}\geq 0$, i.e., $0 \leq \theta_{i_1}<\omega_{K(i_2)}<\omega_{K(i_1)} \leq \pi$ and $0<\theta_{i_2}<\omega_{K(i_2)}$
		\item 2). $\xi_{i_1, K(i_1)}\geq 0$ and $\xi_{i_2, K(i_2)}<0$, i.e., $0 \leq \theta_{i_1}<\omega_{K(i_1)}\leq \pi$ and $\pi+\theta_{i_2}<\omega_{K(i_2)}<2\pi$ and $0 \leq \theta_{i_2}<\pi$
		\item 3). $\xi_{i_1, K(i_1)}<0$ and $\xi_{i_2, K(i_2)}\geq 0$, i.e., $0 \leq \theta_{i_1}<\theta_{i_2}<\pi$ and $\pi<\omega_{K(i_1)}<\omega_{K(i_2)}\leq \theta_{i_2}+\pi$
		\item 4). $\xi_{i_1, K(i_1)}<0$ and $\xi_{i_2, K(i_2)}<0$, i.e., $0 \leq \theta_{i_1}<\theta_{i_2}<\pi$ and $\pi+\theta_{i_2} \leq \omega_{K(i_2)}<\omega_{K(i_1)}<2\pi$
	\end{itemize}
	
	We show the exchange of $K(i_1) \to K(i_2)$ and $K(i_2) \to K(i_1)$ makes the optimized value of (\ref{circle1}) decrease, because it decreases the following term:
	\begin{equation}
	\label{local-change}
	d^2_{\Gamma} (r_{i_1 l}, \mu_{K(i_1)}) +d^2_{\Gamma} (r_{i_2 l}, \mu_{K(i_2)}),
	\end{equation}
	and it suffices to consider $\xi^2_{i_1, K(i_1)}+\xi^2_{i_2, K(i_2)}$ since they are proportional.\\
	
	First consider case 1) and 4), which are equivalent by a shifting and reflection. For case 1), we have $\xi^2_{i_1, K(i_1)}+\xi^2_{i_2, K(i_2)} = {(\omega_{K( i_2 )}-\theta_{ i_2})}^2+{(\omega_{K( i_1)}-\theta_{ i_1})}^2$.
		\begin{equation}
		\begin{aligned}
		& {(\omega_{K(i_2)}-\theta_{i_2})}^2 +{(\omega_{K(i_1)}-\theta_{i_1})}^2 \\
		&-{(\omega_{K(i_1)}-\theta_{i_2})}^2 -{(\omega_{K(i_2)}-\theta_{i_1})}^2 \\
		&= 2(\theta_{i_2} - \theta_{i_1})(\omega_{K(i_1)}- \omega_{K(i_2)}) >0
		\end{aligned} 
		\end{equation}
		Thus cases 1) and 4) are proved. Next consider cases 2) and 3), which are equivalent also. For case 3), we have $\xi^2_{i_1, K(i_1)}+\xi^2_{i_2, K(i_2) }= {(\omega_{K(i_2)}-\theta_{i_2})}^2 +{(\omega_{K( i_1)}-\theta_{ i_1}-2\pi)}^2$. Now the exchange of $i_1$ and $i_2$ results in:
			\begin{equation}
			\begin{aligned}
			& {(\omega_{K(i_2)}-\theta_{i_2})}^2 +{(\omega_{K( i_1)}-\theta_{ i_1}-2\pi)}^2 \\
			&-{(\omega_{K(i_1)}-\theta_{i_2})}^2 -{(\omega_{K( i_2)}-\theta_{ i_1}-2\pi)}^2 \\
			&= 2(\theta_{i_2} - \theta_{i_1})(\omega_{K(i_1)}- \omega_{K(i_2)}) +4\pi(\omega_{K( i_2)}-\omega_{K( i_1)})>0
			\end{aligned} 
			\end{equation}
			where the second line comes from the fact that $\pi<\omega_{K(i_2)}\leq \theta_{i_2}+\pi<2\pi$ .
			\end{proof}

\begin{figure*}
\centering
\includegraphics[width=2 in,height=2 in]{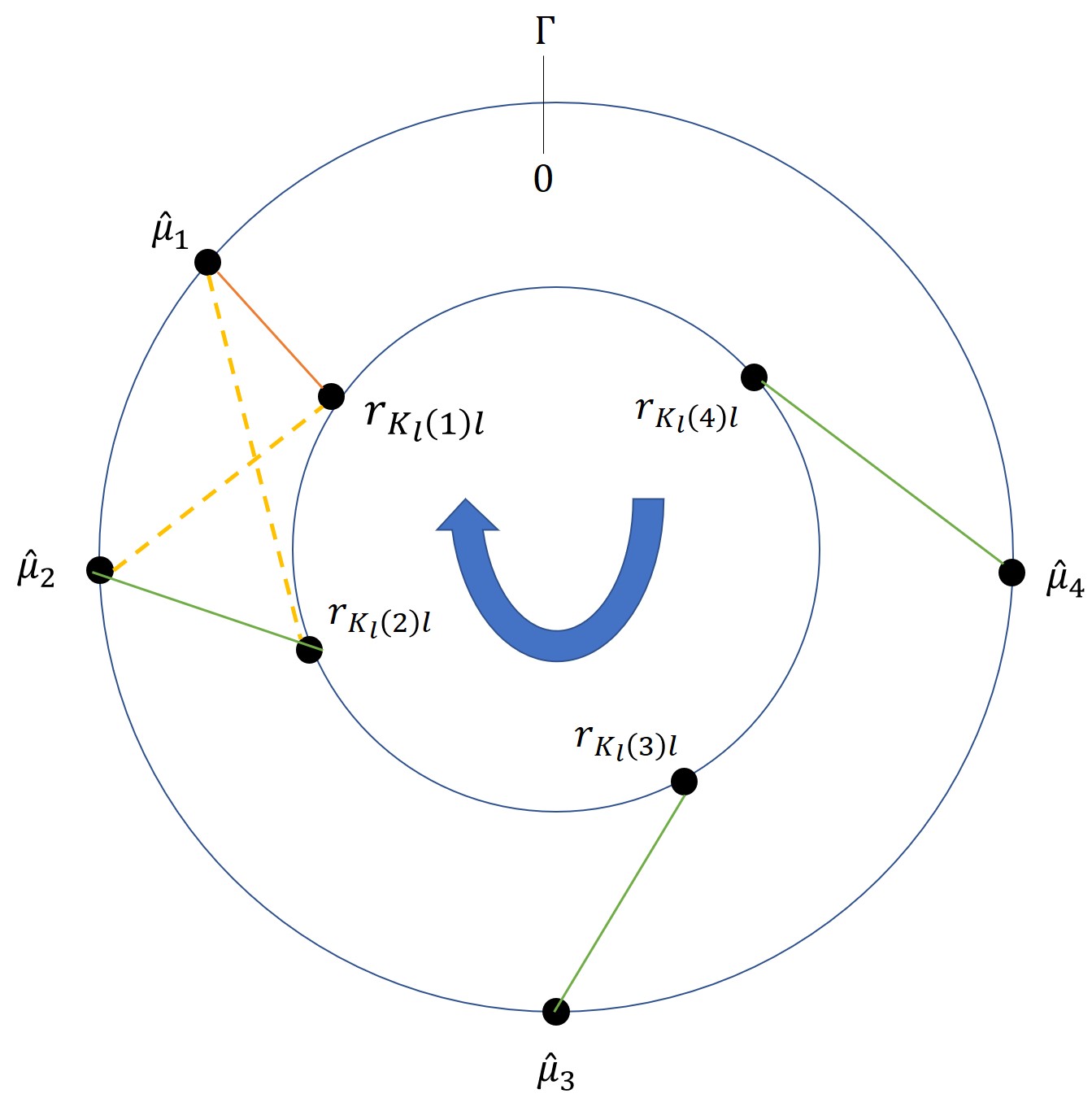}
\centering
\includegraphics[width=2 in,height=2 in]{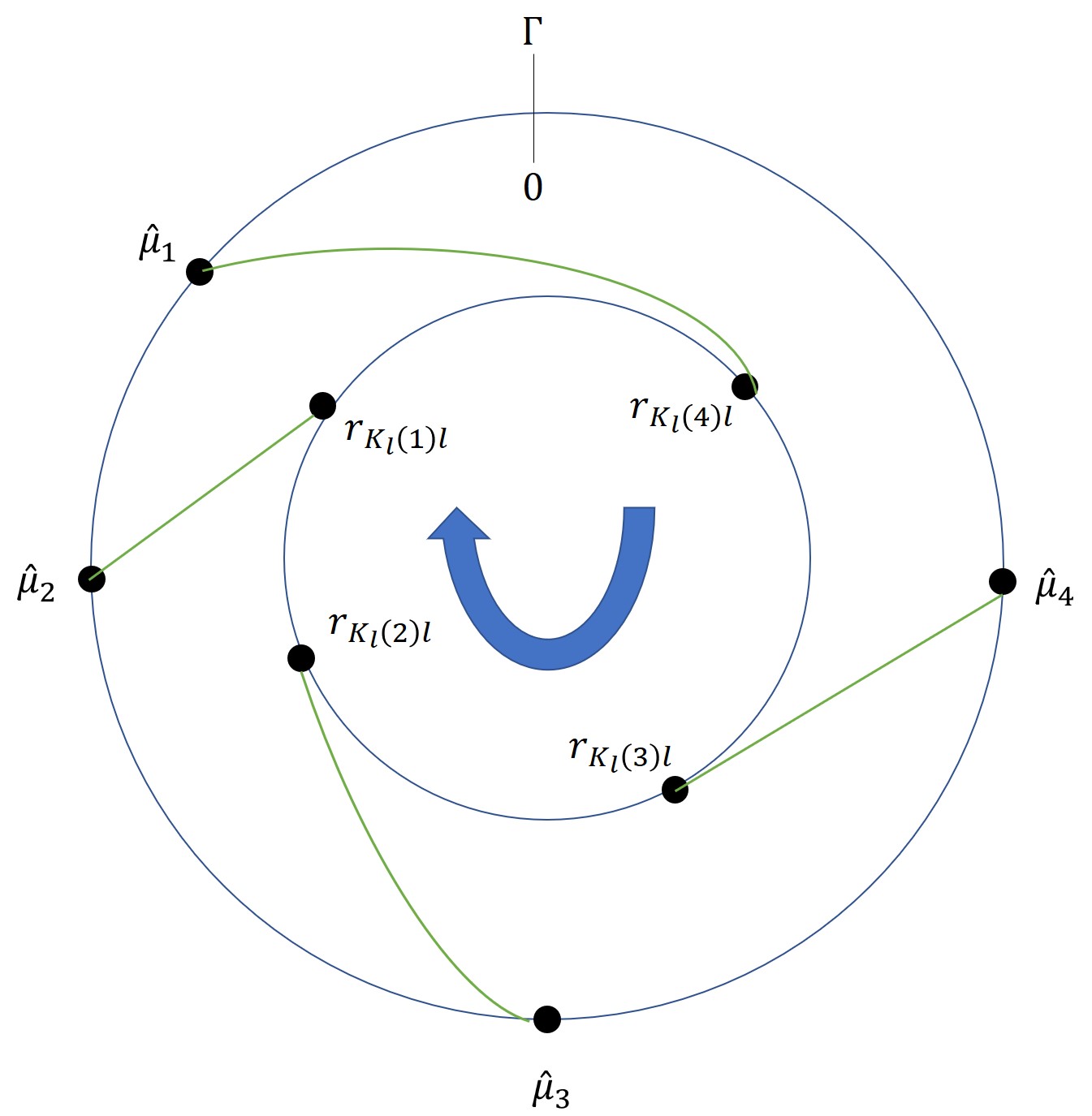}
\caption{Illustration for the Step-1 of Algorithm 2}
\label{concentric}
\end{figure*}

We now move to the second step: knowing the clustering, how to figure out the optimal common residue? It is evident that with the clustering, such estimation is reduced to $N$ independent estimation for a single common residue. This problem has been previously studied in \cite{mle1} where it is proved that the optimal estimation can be determined in $O(L)$ complexity. For completeness, we present the skeleton of \cite{mle1} as follows with a shorter proof.

With given clustering $K_l$, $l=1,2,...,L$, to figure out the optimal $\bm{\hat{\mu}}$, where 
\begin{equation}
\hat{\mu_i} = arg \min_{x \in [0,\Gamma)} \sum_{l=1}^{L} d^2_{\Gamma} (x, \mu_{k_l(i)l}) 
\end{equation}
For simplicity, we assume that $\gamma_1, \gamma_2, ... ,\gamma_L$ denote $\mu_{k_l(i)l}$ in an ascending order. In addition, we define that $\gamma_{\langle l -1 \rangle_{L}}$ and $\gamma_{\langle l +1 \rangle_{L}}$ are neighbors of $\gamma_l$ in the sense over the circle modulo $\Gamma$. Therefore, there must exist a binary variable $b_{l} \in \{0,1\}$, such that 
\begin{equation}
 \sum_{l=1}^{L} w_{l} d^2_{\Gamma} (\hat{\mu_i} , \mu_{k_l(i)l}) = \sum_{l=1}^{L} w_l (\hat{\mu_i} - \gamma_l+b_{l}\Gamma)^2
\end{equation}
since $d_{\Gamma} (\hat{\mu_i} , \gamma_l) \leq \frac{\Gamma}{2}$ in the optimal case and both $\hat{\mu_i}$ and $\gamma_l \in [0,\Gamma)$. Thus $d^2_{\Gamma} (\hat{\mu_i} , \mu_{k_l(i)l})$ must fall in one of $\{(\hat{\mu_i} -\mu_{k_l(i)l})^2, (\hat{\mu_i} - \mu_{k_l(i)l}-\Gamma)^2 \}$.  In the following, we show that for the optimal case, there must exist some $l_0 \in \{1,2,...,L\}$ such that for $l<l_0$, $b_l=1$ and for $l \geq l_0$, $b_l =0$. Since the distance between each $\hat{\mu_i}$ and $\gamma_l+b_{l}\Gamma$ is no bigger than $\frac{\Gamma}{2}$, therefore the absolute difference between any two of $\gamma_l+b_{l}\Gamma$ should also be no bigger than $\Gamma$.  First, it is not hard to observe that for any two $l_1 < l_2 \in \{1,2,...,L\}$, $b_{l_1} \geq b_{l_2}$. Otherwise, if  $b_{l_2}=1$ while $b_{l_1}=0$, $\gamma_{l_2}+b_{l_2}\Gamma - \gamma_{l_1}+b_{l_1}\Gamma = \gamma_{l_2}- \gamma_{l_1} + \Gamma > \Gamma.$ With the non-increasing property of $b_l$, the claim follows clearly. Therefore, $\textbf{b} = ( b_1, b_2, ... ,b_L)$ must fall in one the following candidates, $\textbf{v}_1 = (0,0, ... ,0,0), \textbf{v}_2 = (1,0, ... ,0,0),...,\textbf{v}_L=(1,1, ... ,1,0) \}$. On the other hand, when $\textbf{b} = \textbf{v}_j$, 
\begin{equation}
\label{M-step}
arg \min_{x \in [0,\Gamma)} \sum_{l=1}^{L} w_l (x-\gamma_l-b_l\Gamma)^2 = \frac{\sum_{l=1}^L w_l \gamma_l + \sum_{l=1}^{j} w_l\Gamma}{L}
\end{equation}
Therefore, given clustering, the optimal estimation of common residues can be determined in $O(L)$ complexity for each, and totally $O(NL)$.

As a summary, with Theorem 2, given an estimation of $\bm{\mu}$, we can figure out the optimal classification $\bm{K}_{[1:L]}$ in $O(NL)$ complexity. Relying on an estimation of $\bm{K}_{[1:L]}$, we can further determine the MAP of $\bm{\mu}$ still in $O(NL)$ complexity. Combing both, the MAP of both can be estimated by an iterative searching with two steps alternatively and we conclude the algorithm as follows.

\begin{algorithm}
\caption{MAP of Classification and Common Residues}
\textbf{Input:} Given moduli $m_l=\Gamma M_l$ and the residue observed $R_{il} $, $i=1,2,...,N$ and $l=1,2,...,L$. 

1. Calculate ${r}_{il} = \langle  \widetilde{r}_{il} \rangle_{M_l}$;

2. Start iteration $t$ starting from $1$ with an initialization for $\{\hat{\mu}_i(0), i=1,2,...,N\}$ by randomly drawing a set of observations, i.e. $\bm{r}_{[1:N],l} = \{r_{1,l}, r_{2,l},...,r{N,l}\}$ if observation set $l$ is drawn.  
3. Begin iteration: 
    \begin{itemize}
     \item Step-1: Under the assumption where common residues are $\{\hat{\mu}_i(t-1), i=1,2,...,N\}$, follow Theorem \ref{circlemin1} to determine the optimal classification, expressed by $\{K^{t}_l, l=1,2,...,L\}$. 
     
     \item Step-2: Under the assumption where the classifications are $\{K^{t}_l, l=1,2,...,L\}$, update $\{\hat{\mu}_i(t), i=1,2,...,N\}$ by solving the following optimization problem:
     \begin{equation}
\hat{\mu}_i(t) = \arg \min_{x \in [0,\Gamma)} \sum_{l=1}^{L} d^2_{\Gamma}(x, \widetilde{r}_{K^{t}_l(i)l})
\end{equation}    
    \item $t = t+1$. 
    \end{itemize}

4. Calculate 
\begin{equation}
    q_{il} = [ \frac{{R}_{K^{T}_l(i)l}-\widetilde{r}_{K^{T}_l(i)l}}{\Gamma} ]
\end{equation}
and reconstruct $Q_i$ from $q_{il}$ with moduli $M_l$ via conventional CRT.

5. Finally reconstruct $ \hat{Y}_i = Q_i\Gamma + \hat{\mu}_i(T)$. 

\textbf{Output:}  $\hat{Y}_i$, $i=1,2,...,N$.
\end{algorithm}

\section{Robustness Strengthening and Simulation}
\noindent Throughout the section, we will introduce error correcting codes to further strengthen the robustness of proposed statistical RCRTs.  As we stress before, once the classification is not perfectly correct, even if only one residue is not correctly clustered, it will ruin the whole estimation. 
A natural question is that when the accuracy of classification is not met with $100 \%$, whether we can still achieve robust construction. Fortunately, one of our previous work has provided a positive answer to this problem, as the mistake in classification can be reduced to a special case of arbitrary errors in residues. Assuming that if $L$ many moduli are used where $m_l=\Gamma M_l$ are in an ascending order, and $L_0$ is the smallest positive integer $L_0 \leq L$, such that $lcm(m_1,m_2,...,m_{L_0}) = \Gamma \prod_{l=1}^{L_0} M_l> D$, then, if no error exists, the residues from $L_0$ moduli are sufficient to recover $Y_i$. Thus we have the following Theorem, 

\begin{thm} [\cite{tvt-2019}]
\label{tvt}
For given $\bm{K}_{[1:N]} = (K_1, K_2, ... ,K_L)$, when at least $L-\lfloor \frac{L-L_0}{2} \rfloor$ many residues $R_{il}$ for $X_i$ are clustered together and for each $i$, $\max_l \Delta_{il} - \min_l  \Delta_{il} < \frac{\Gamma}{2}$, then there exists a robust reconstruction scheme with output $\hat{Y}_i$ such that $|\hat{Y}_i - Y_i| \leq \frac{3\Gamma}{4}$ in $O(L-L_0+1)$ times error correction. 
\end{thm}

It is worthy mentioning that with fixed $L_0$, increasing $L$, i.e., more moduli(samples), can not be guaranteed to continue bringing benefits as a larger $L$ also degrades the classification performance. Moreover, it is also required that to achieve the robustness, beyond the correct clustering, the residue errors should also be spanned in an interval with maximum length $\frac{\Gamma}{2}$. \footnote{A stronger bound is shown in \cite{mle1} with respect to different weights. Here we just assume the weights are the same for brevity.}  As $L$ increases, the probability that the span of all errors is smaller than $\frac{\Gamma}{2}$ can be referred to (4) by setting $\delta = \frac{\Gamma}{4}$.

Such trade off, depending on $N$, $L$, $L_0$, the noise and also the desirable computation power, is too complicated to be concisely expressed but when the noise is not too large, adding redundancy properly can always improve the performance. Besides, the other advantage of the error-correction mechanism will be clear for the following majority-based estimation. 

In the rest of the section, we will show the mechanism to fully utilize the data from multiple samplers. A natural idea is that we can regroup the moduli into many sets. Then we implement Algorithm 1 or 2 on data aggregated from each set. Basically, the minimal requirement is that the data in each set should be from at least $L_{\min}$ moduli, the lcm of which is bigger than $D$ to make the final reconstruction valid. Thus from each set, it can output a group of estimations $\{\hat{Y}_i\}$. If we have $\kappa$ many such moduli sets, correspondingly we can then pick the $N$ estimations with top appearing frequency from all $\kappa N$ reconstructed numbers. It is noted that if we take all subsets of moduli of size $S \geq L_{\min}$, $\kappa \leq \binom{L}{S}$. When $2S<L$, a larger $S$ can bring more estimations and Theorem \ref{tvt} plays a key role in the trade off between $S$ and the performance for each estimation.

\begin{figure*}
\label{simu}

\centering
\subfigure[Proposed Statistical RCRT-1 and Deterministic RCRT \cite{TSP2018}]{

\begin{minipage}[t]{0.48\textwidth}
\centering
\centerline{\includegraphics[width=90mm]{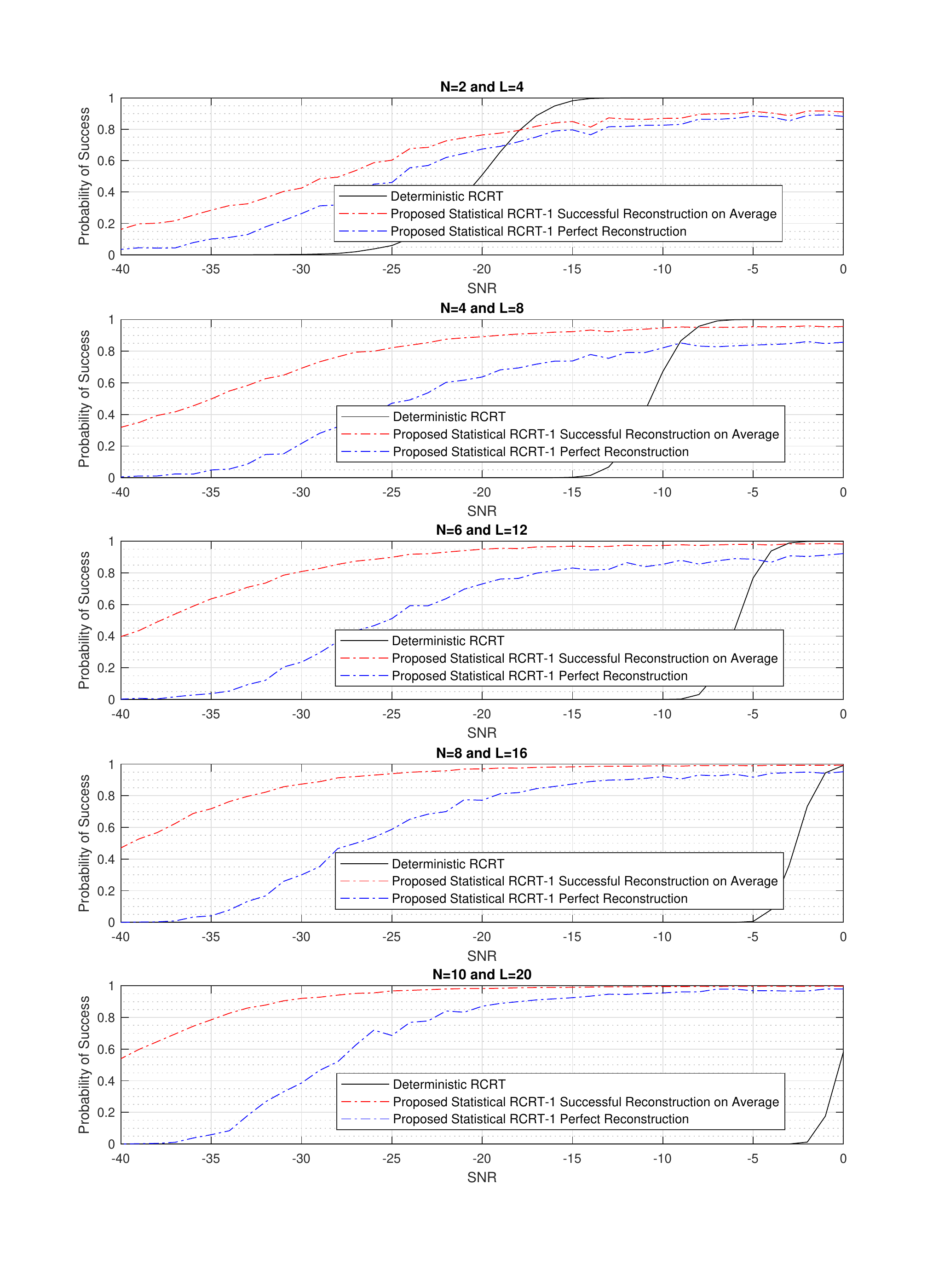}}
\end{minipage}
}
\subfigure[Proposed Statistical RCRT-2 and Deterministic RCRT \cite{TSP2018}]{
\begin{minipage}[t]{0.48\textwidth}
\centering
\centerline{\includegraphics[width=90mm]{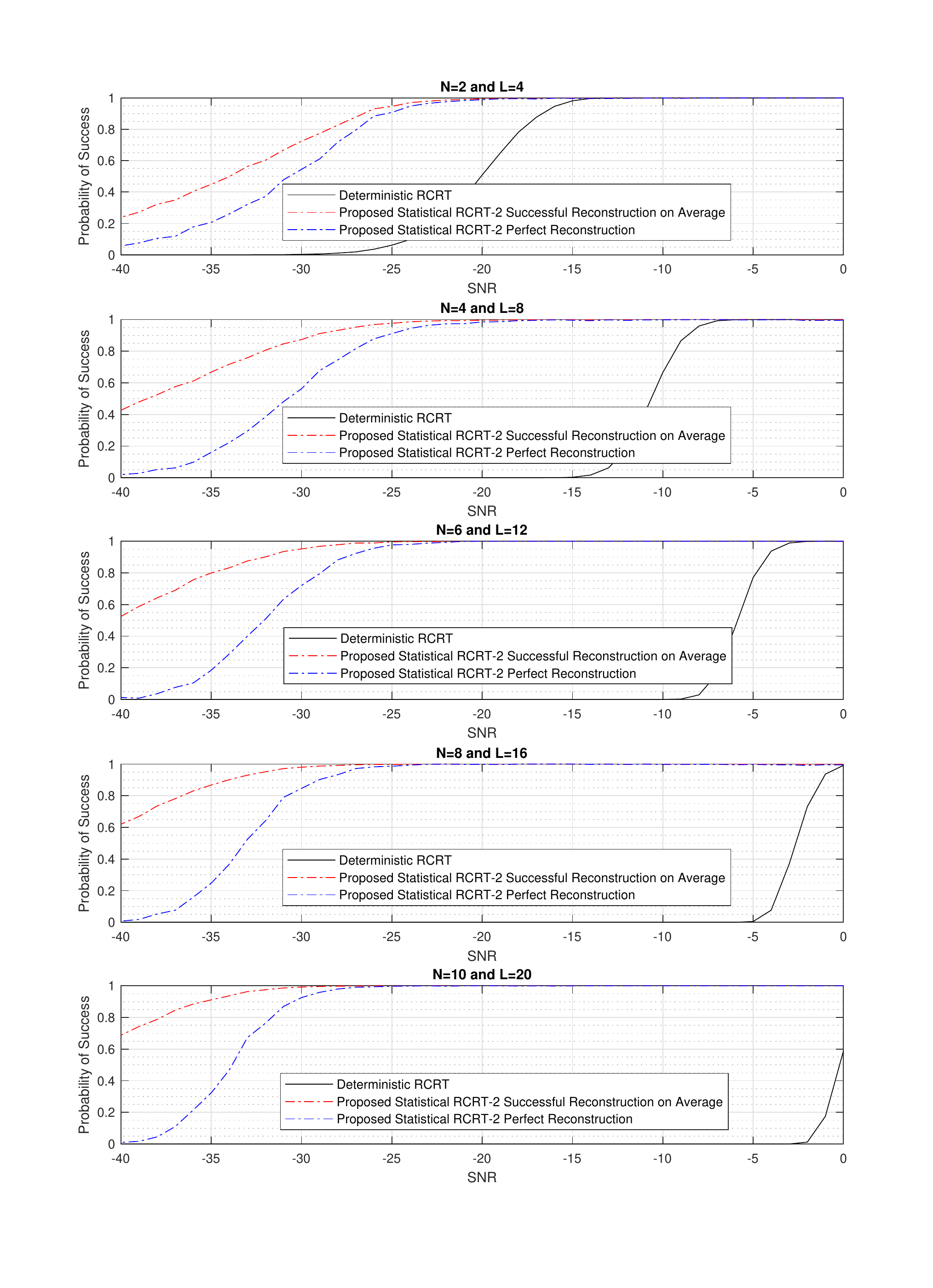}}
\end{minipage}
}
\caption{Comparison for Success Rate of Robust Reconstruction between the proposed two statistical RCRT and Deterministic RCRT in \cite{TSP2018}}
\end{figure*}

Whereas, for two different sets, even both with perfect classification, noise may still cause a sight difference between two estimations for the same $Y_i$, which prohibits using the above idea straightforwardly. However, since the final reconstruction is formed by two parts, one to recover the quotient and the other to recover the common residues. In our protocol, we focus on the $N$ quotients with highest frequency. \footnote{Even under perfect classification, with the proposed scheme, the common residues may be uniformly shifted by $\Gamma$, depending on the choice of the cutting point $\tau$, and correspondingly, two reconstructions on the quotient $\lfloor \frac{X_i}{\Gamma} \rfloor$ may differ by $1$. However, it will not affect the final reconstruction and we can unite the two cases into one and assume they share the same estimated quotient.}  When the moduli are divided into $\kappa$ subsets without intersection, then the estimations from different groups are independent. Let $B_j$ denote the event that $\lfloor \frac {Y_i}{\Gamma} \rfloor$ can be robustly reconstructed from the $j^{th}$ group and $\Pr(B_j)$ is should be a constant. Taking $B_j$ as a Bernoulli variable, then $\sum_{j=1}^{\kappa} B_j > \frac{\gamma}{2}$ is sufficient to show that the correct reconstruction of  $\lfloor \frac{X_i}{\Gamma} \rfloor$ is one of the estimation with highest $N$ frequencies. From Chernoff bound, 

\begin{lm}[Chernoff Bound for Bernoulli Variables]  
If $E[B_j]=p>\frac{1}{2}$ for $j=1,2,...,\kappa$, then 
\begin{equation}
\label{chernoff}
\Pr( \sum_{j=1}^{n} B_j > \frac{\kappa}{2}) \geq 1-e^{-\frac{\kappa}{2p} (p-\frac{1}{2})^2},
\end{equation}
\end{lm}
\noindent the failure probability is exponentially decaying. Due to the symmetry, the bound also works for any $Y_i$. However, the actual performance is much stronger than this bound. Heuristically, when classification fails, the resulted $\hat{Y}_i$ is almost uniform over $[0,D)$. Moreover, both with faulted classification, the probability that two different groups can output a same quotient estimation is also very limited, especially when $M_l$ is large. Therefore, when $Y_i$ has been robustly recovered at least twice, it is already with considerable possibility of success to be selected in the output. The following simulation results capture such intuition well.

For comparison, here we present a stronger deterministic RCRT for multiple numbers as an extended version of that in \cite{TSP2018}. Rather than assuming that $|\max_{il} \Delta_{il}| < \frac{\Gamma}{4N}$, it can be proved that all the conclusions and algorithms in \cite{TSP2018} hold if $|\max_{l} \Delta_{il} - \min_{l} \Delta_{il}| < \frac{\Gamma}{2N}$ for each $i$. 
The following simulation results show the comparison among Algorithm 1(MAP of classification), Algorithm 2 (MAP of both classification and common residues) and improved deterministic RCRT shown above. Here we assume that $L_{\min}=2$. Moduli are selected as $\Gamma=100$ and $M_1, M_2, ... ,M_L$ are the sequence of primes starting from 21. Referring to the requirement of moduli in \cite{TSP2018}, the lcm of all moduli is in the degree of product of $Y_i$ and we set $L=L_{\min}N=2N$. We assume $\Gamma =100$ and the variance of noise $\sigma^2 = 10^{-SNR/10}$. As for the simulation shown in Fig. 3, for both Algorithm 1 \& 2, we utilize residues from each pair of moduli, in total $\binom{2N}{2}$ many groups as tests. We define a successful robust reconstruction here as the final reconstruction error is upper bounded by $\Gamma$. For each SNR being integers within $[-40, 0]$ and $N$ ranging from $2, 4, 6, 8,10$, we run 1000 simulations to estimate the success rate in each scenario.  Here we provide two metrics of the success rate, where one is termed as the successful reconstruction on average and the other is called the perfect reconstruction. Since in both proposed statistical RCRT, the final reconstruction is finally divided into $N$ independent reconstruction processes for each $Y_i$, the average success rate denotes that the expectation of a $Y_i$ can be robustly recovered. Similarly, the perfect reconstruction denotes that all $Y_i$ are robustly recovered. The two metrics may be of different interests in different applications. A comparison between Algorithm 1 proposed and deterministic RCRT in \cite{TSP2018} is presented in Fig. 3 (a) and similarly the comparison between Algorithm 2 and that in  \cite{TSP2018} is shown in Fig.3 (b). From Fig. 3, Algorithm 2 outperformances Algorithm 1 while as analyzed before, Algorithm 2 is iterative based, which may face a little more computational overhead. We also include the number of iterations of  Algorithm 2 on average in Fig. 4. The five subfigures in Fig. 4 show the relationship between $N$ and the distribution of iteration numbers and for $N=2,4,6,8$ and $10$, how the noise affects the iteration. Here ''low scenario'' denotes the cases where SNR is within $[-40,-20)$ and ''high scenario'' refers to SNR within $[-20,-0]$. In general, within $10$ iteration, Algorithm 2 can reach a local minimal state. At last, we give two examples with respect to how error-correcting techniques can improve or lead to a sharpened trade off. In Fig. 5, when $N=2$ and $L=4$, we try estimation only once with all four moduli but incorporated with error-correction.\footnote{From Theorem \ref{tvt}, here we can tolerate at most one clustering error since $L_{min}=2$ and $\lfloor \frac{L-L_{\min}}{2} \rfloor =1$.} From the first subfigure in Fig. 5, though with some compromise in performance, we implement Algorithm 2 only once. In the second subgraph, to be fair, we test when $N=6$, randomly selecting $\binom{L}{2}$ many subsets of moduli of size $4$ but applying error-correction. It shows that the error-correcting based Algorithm 2 outperformances the previous one when SNR is bigger than -37.



\begin{figure}
\label{iteration}
	\vskip 0.2in
	\begin{center}
		\centerline{\includegraphics[width=90mm]{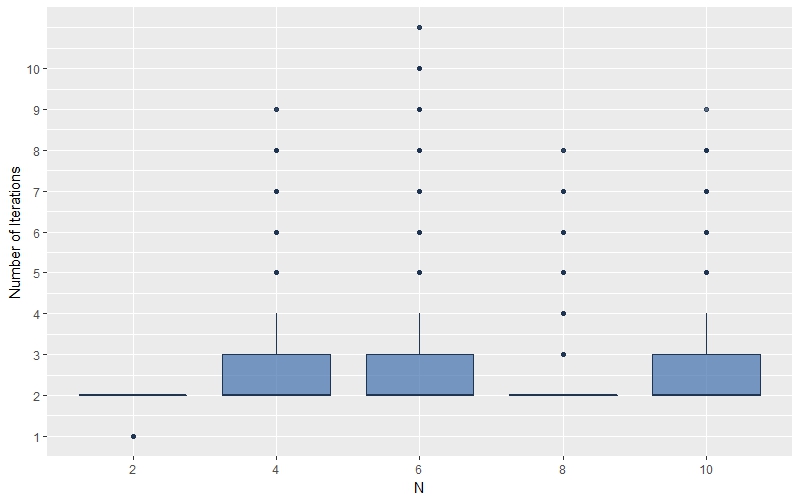}}
		\centerline{\includegraphics[width=90mm]{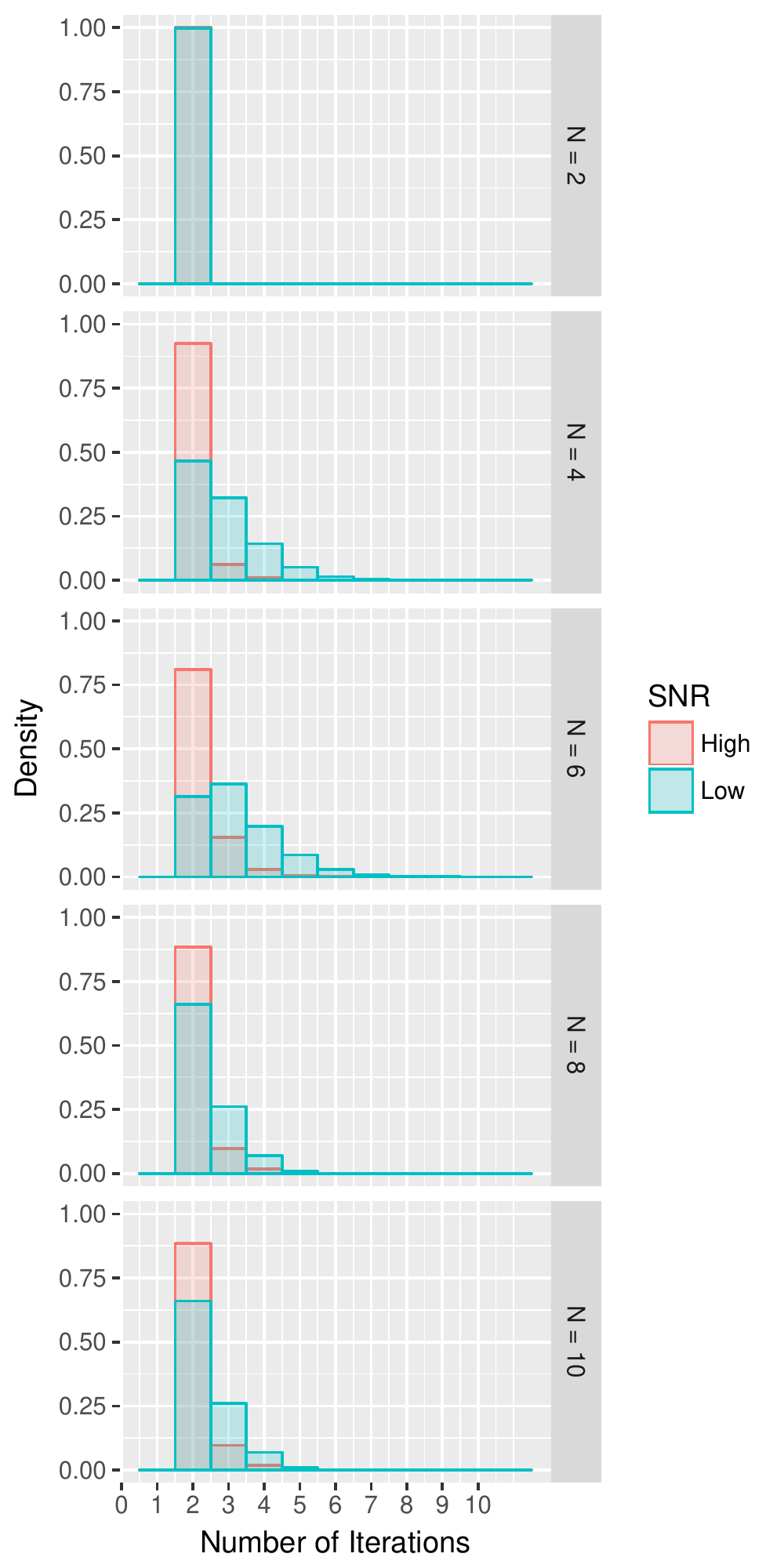}}
		\caption{ Number of Iterations of Algorithm 2 on Average }
		\label{MLEsim}
	\end{center}
	\vskip -0.2in
\end{figure}

\begin{figure}
\label{thres}
	\vskip 0.2in
	\begin{center}
		\centerline{\includegraphics[width=90mm]{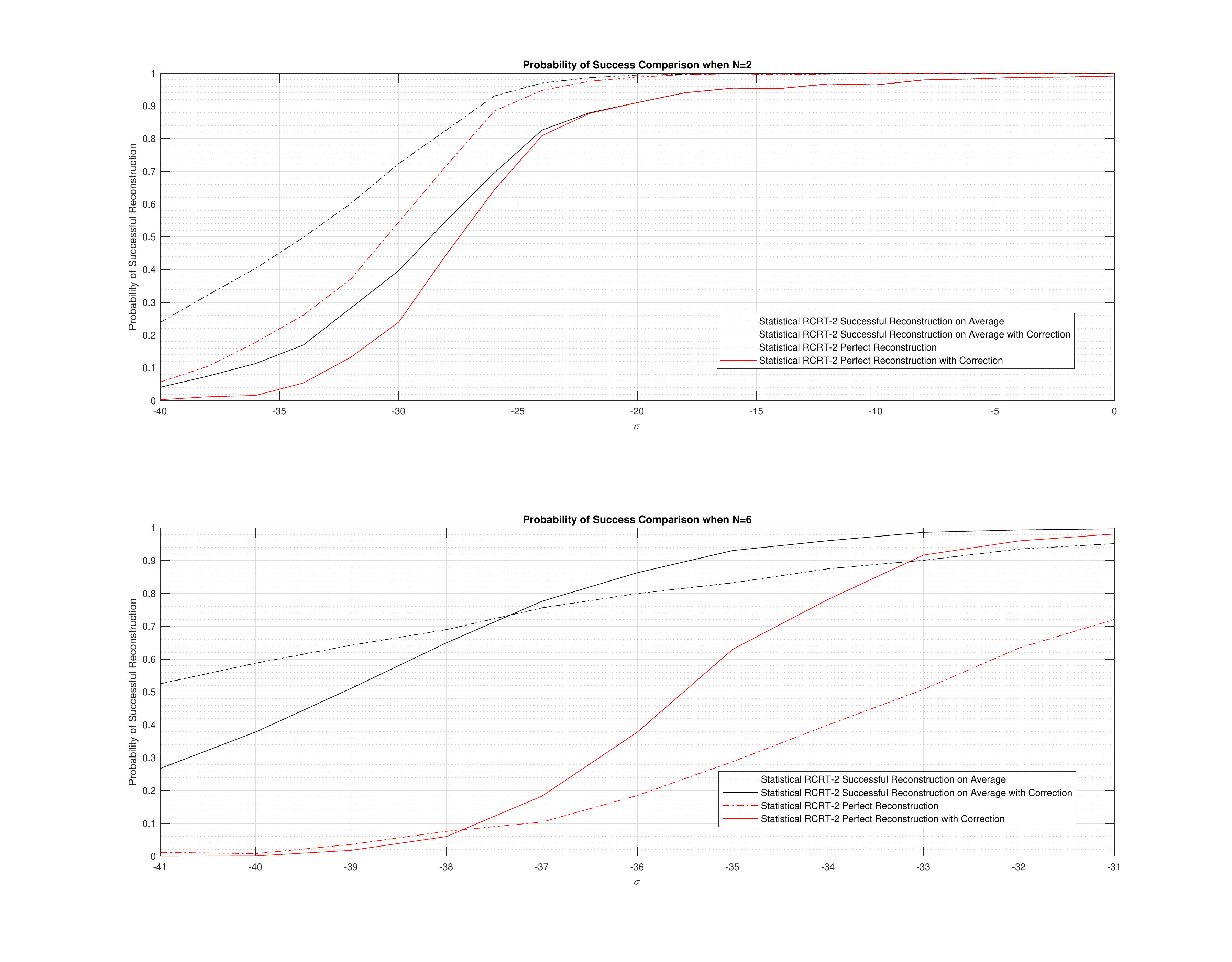}}
		\caption{ Proposed Statistical RCRT-2 with Error Correction }
		\label{MLEsim}
	\end{center}
	\vskip -0.2in
\end{figure}

\section{Conclusion and Prospect}
\noindent In this paper, we present the first statistical based approaches to efficiently solve the robust reconstruction of multiple numbers from unordered residues. Different from deterministic schemes restricted by unique error syndrome detection, the proposed two statistical RCRT overcomes sensitivity both to noise and the number of moduli. Such improvement is significant especially in low and median SNR circumstance and further strengthened with error-correcting techniques. However, as analyzed before, when noise decreases, the probability that the errors meet the feasible correction range will gradually catch up and exceed the success rate of proposed algorithms. Therefore, it would be of great interest to explore the intersection between statistical inference and deterministic error tolerance. Another problem remains open is that what is the optimal selection of the size for the test subset. Roughly, it is nontrivial to determine whether samples from more moduli used in each inference step can benefit or not. When samples from more moduli are used, the clustering accuracy degrades but incorporated with error-correcting codes, the robustness with respect to tolerance of larger burst errors and even a wrong classification error is improved. Besides, more test sets can be constructed. Therefore, the trade off between parameters selection is worthy further investigation.

\bibliographystyle{plain}
\bibliography{ref}

\begin{thebibliography}{10}

\bibitem{boneh2002}
Dan Boneh.
\newblock Finding smooth integers in short intervals using crt decoding.
\newblock {\em Journal of Computer and System Sciences}, 64(4):768--784, 2002.

\bibitem{goldreich1999}
Oded Goldreich, Dana Ron, and Madhu Sudan.
\newblock Chinese remaindering with errors.
\newblock In {\em Proceedings of the thirty-first annual ACM symposium on
  Theory of computing}, pages 225--234. ACM, 1999.

\bibitem{guruswami2000}
Venkatesan Guruswami, Amit Sahai, and Madhu Sudan.
\newblock " soft-decision" decoding of chinese remainder codes.
\newblock In {\em Foundations of Computer Science, 2000. Proceedings. 41st
  Annual Symposium on}, pages 159--168. IEEE, 2000.

\bibitem{inequalities}
Godfrey~Harold Hardy, John~Edensor Littlewood, and George P{\'o}lya.
\newblock {\em Inequalities}.
\newblock Cambridge university press, 1988.

\bibitem{jenkins1984technique}
W~Jenkins.
\newblock A technique for the efficient generation of projections for error
  correcting residue codes.
\newblock {\em IEEE transactions on circuits and systems}, 31(2):223--226,
  1984.

\bibitem{krishna1993theory}
Hari Krishna and J-D Sun.
\newblock On theory and fast algorithms for error correction in residue number
  system product codes.
\newblock {\em IEEE Transactions on Computers}, (7):840--853, 1993.

\bibitem{SAR2}
Gang Li, Jia Xu, Ying-Ning Peng, and Xiang-Gen Xia.
\newblock Bistatic linear antenna array sar for moving target detection,
  location, and imaging with two passive airborne radars.
\newblock {\em IEEE Transactions on Geoscience and Remote Sensing},
  45(3):554--565, 2007.

\bibitem{SAR3}
Gang Li, Jia Xu, Ying-Ning Peng, and Xiang-Gen Xia.
\newblock Location and imaging of moving targets using nonuniform linear
  antenna array sar.
\newblock {\em IEEE Transactions on Aerospace and Electronic Systems}, 43(3),
  2007.

\bibitem{mle2}
Xiaoping Li, Wenjie Wang, Weile Zhang, and Yunhe Cao.
\newblock Phase-detection-based range estimation with robust chinese remainder
  theorem.
\newblock {\em IEEE Transactions on Vehicular Technology}, 65(12):10132--10137,
  2016.

\bibitem{RCRT-two}
Xiaoping Li, Xiang-Gen Xia, Wenjie Wang, and Wei Wang.
\newblock A robust generalized chinese remainder theorem for two integers.
\newblock {\em IEEE Transactions on Information Theory}, 62(12):7491--7504,
  2016.

\bibitem{li2009fast}
Xiaowei Li, Hong Liang, and Xiang-Gen Xia.
\newblock A robust chinese remainder theorem with its applications in frequency
  estimation from undersampled waveforms.
\newblock {\em IEEE Transactions on Signal Processing}, 57(11):4314--4322,
  2009.

\bibitem{li2008fast}
Xiaowei Li and Xiang-Gen Xia.
\newblock A fast robust chinese remainder theorem based phase unwrapping
  algorithm.
\newblock {\em IEEE Signal Processing Letters}, 15:665--668, 2008.

\bibitem{dynmaic-sharpen}
Huiyong Liao and Xiang-Gen Xia.
\newblock A sharpened dynamic range of a generalized chinese remainder theorem
  for multiple integers.
\newblock {\em IEEE transactions on information theory}, 53.

\bibitem{rns1963}
Jeremy~J Stone.
\newblock Multiple-burst error correction with the chinese remainder theorem.
\newblock {\em Journal of the Society for Industrial and Applied Mathematics},
  11(1):74--81, 1963.

\bibitem{wang2011robust}
Chen Wang, Qinye Yin, and Hongyang Chen.
\newblock Robust chinese remainder theorem ranging method based on
  dual-frequency measurements.
\newblock {\em IEEE Transactions on Vehicular Technology}, 60(8):4094--4099,
  2011.

\bibitem{SAR1}
Genyuan Wang, Xiang-Gen Xia, Victor~C Chen, and RL~Fielder.
\newblock Detection, location, and imaging of fast moving targets using
  multifrequency antenna array sar.
\newblock {\em IEEE Transactions on Aerospace and Electronic Systems},
  40(1):345--355, 2004.

\bibitem{dynmaic-two}
Wei Wang, Xiaoping Li, Xiang-Gen Xia, and Wenjie Wang.
\newblock The largest dynamic range of a generalized chinese remainder theorem
  for two integers.
\newblock {\em IEEE Signal Processing Letters}, 22(2):254--258, 2015.

\bibitem{mle1}
Wenjie Wang, Xiaoping Li, Wei Wang, and Xiang-Gen Xia.
\newblock Maximum likelihood estimation based robust chinese remainder theorem
  for real numbers and its fast algorithm.
\newblock {\em IEEE Transactions on Signal Processing}, 63(13):3317--3331,
  2015.

\bibitem{closed}
Wenjie Wang and Xiang-Gen Xia.
\newblock A closed-form robust chinese remainder theorem and its performance
  analysis.
\newblock {\em IEEE Transactions on Signal Processing}, 58(11):5655--5666,
  2010.

\bibitem{dynamic-1}
X-G Xia.
\newblock An efficient frequency-determination algorithm from multiple
  undersampled waveforms.
\newblock {\em IEEE Signal Processing Letters}, 7(2):34--37, 2000.

\bibitem{xia2007phase}
Xiang-Gen Xia and Genyuan Wang.
\newblock Phase unwrapping and a robust chinese remainder theorem.
\newblock {\em IEEE Signal Processing Letters}, 14(4):247--250, 2007.

\bibitem{dynamic-2}
Hanshen Xiao, Cas Cremers, and Hari~Krishna Garg.
\newblock Symmetric polynomial \& crt based algorithms for multiple frequency
  determination from undersampled waveforms.
\newblock In {\em Signal and Information Processing (GlobalSIP), 2016 IEEE
  Global Conference on}, pages 202--206. IEEE, 2016.

\bibitem{TSP2018}
Hanshen Xiao, Yufeng Huang, Yu~Ye, and Guoqiang Xiao.
\newblock Robustness in chinese remainder theorem for multiple numbers and
  remainder coding.
\newblock {\em IEEE Transactions on Signal Processing}, 66(16):4347--4361,
  2018.

\bibitem{sp2017}
Hanshen Xiao and Guoqiang Xiao.
\newblock Notes on crt-based robust frequency estimation.
\newblock {\em Signal Processing}, 133:13--17, 2017.

\bibitem{tvt-2019}
Hanshen Xiao and Guoqiang Xiao.
\newblock On solving ambiguity resolution with robust chinese remainder theorem
  for multiple numbers.
\newblock {\em arXiv preprint arXiv:1807.00022}, 2018.

\bibitem{towards}
Li~Xiao, Xiang-Gen Xia, and Haiye Huo.
\newblock Towards robustness in residue number systems.
\newblock {\em IEEE Transactions on Signal Processing}, 65(6):1497--1510, 2017.

\bibitem{multistage}
Li~Xiao, Xiang-Gen Xia, and Wenjie Wang.
\newblock Multi-stage robust chinese remainder theorem.
\newblock {\em IEEE Transactions on Signal Processing}, 62(18):4772--4785,
  2014.

\bibitem{xu2014solving}
Guangwu Xu.
\newblock On solving a generalized chinese remainder theorem in the presence of
  remainder errors.
\newblock {\em arXiv preprint arXiv:1409.0121}, 2014.

\bibitem{SAR4}
Jia Xu, Zu-Zhen Huang, Zhi-Rui Wang, Li~Xiao, Xiang-Gen Xia, and Teng Long.
\newblock Radial velocity retrieval for multichannel sar moving targets with
  time--space doppler deambiguity.
\newblock {\em IEEE Transactions on Geoscience and Remote Sensing},
  56(1):35--48, 2018.

\end{thebibliography}

\end{document}